%% file: main.tex
\documentclass{article}
\PassOptionsToPackage{table,dvipsnames}{xcolor}
\usepackage{iclr2027_conference,times}
\iclrfinalcopy
\usepackage[utf8]{inputenc}
\usepackage[T1]{fontenc}
\usepackage[hidelinks]{hyperref}
\usepackage{threeparttable}
\usepackage[table]{xcolor} 
\usepackage{siunitx}
\usepackage{wrapfig}
\usepackage{url}
\usepackage{makecell}
\usepackage{multirow}
\usepackage{booktabs}
\usepackage{amsfonts}
\usepackage{nicefrac}
\usepackage{microtype}
\usepackage{tabularx}
\usepackage{booktabs}
\usepackage{xcolor}
\usepackage{amsmath}
\usepackage{amssymb}
\usepackage{amsthm}
\usepackage{bm}
\usepackage{enumitem}
\usepackage{algorithm}
\usepackage{algpseudocode}
\usepackage{graphicx}
\usepackage{subcaption}
\usepackage{tikz}
\usetikzlibrary{shapes.geometric, arrows.meta, positioning, calc, backgrounds, fit,bayesnet}

\definecolor{LightCyan}{rgb}{0.88,1,1}

\newcommand{\reforgeicon}{\includegraphics[height=22pt,trim=127bp 141bp 119bp 210bp,clip]{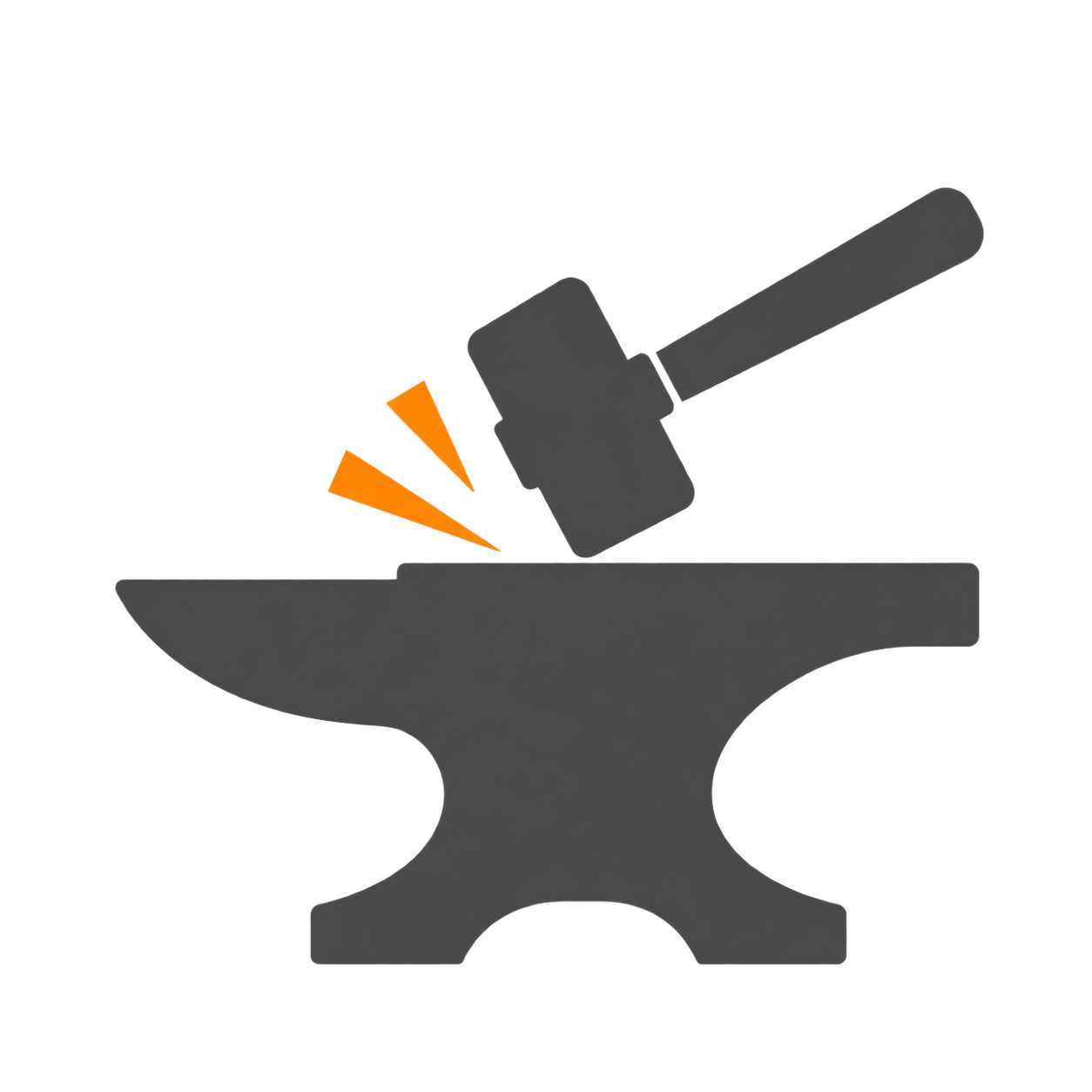}}
\title{\texorpdfstring{\raisebox{-4pt}[16pt][0pt]{\reforgeicon}\hspace{8pt}\textcolor{black}{ReForge: Refining Merged Models\\with Anchor-Regularized Regression}}{ReForge: Refining Merged Models with Anchor-Regularized Regression}}
\author{
  Kaiyang Li$^1$ \hspace{2em}
  Shaobo Han$^2$ \hspace{2em}
  Qing Su$^1$ \hspace{2em}
  Shihao Ji$^1$ \\
  $^1$School of Computing, University of Connecticut, Storrs, CT 06269 \\
  \texttt{\{kaiyang.li, qing.2.su, shihao.ji\}@uconn.edu} \\
  $^2$Department of Machine Learning, NEC Labs America, Princeton, NJ 08540 \\
  \texttt{shaobo@nec-labs.com}
}
\begin{document}
\newlength{\originaltextfloatsep}
\newlength{\originalfloatsep}
\newlength{\originalintextsep}
\setlength{\originaltextfloatsep}{\textfloatsep}
\setlength{\originalfloatsep}{\floatsep}
\setlength{\originalintextsep}{\intextsep}
\setlength{\textfloatsep}{10pt plus 2pt minus 2pt}
\setlength{\floatsep}{8pt plus 2pt minus 2pt}
\setlength{\intextsep}{8pt plus 2pt minus 2pt}

\maketitle
\lhead{}
\begin{abstract}
Model merging aims to combine multiple task-specific expert models into a single model without joint retraining, offering a practical alternative to multi-task learning when data access or computational budget is limited. \textcolor{black}{Existing model merging methods rarely exploit strong merged models as priors for further improvement.} \textcolor{black}{To address this limitation, we propose \textbf{ReForge}, a bilevel optimization framework that formulates module-wise refinement as Bayesian linear regression with an anchor-centered prior. The inner level yields a closed-form MAP estimate from unlabeled calibration activations. The outer level uses Bayesian optimization to jointly select heterogeneous regularization strengths and assembly scales using held-out validation data.} \textcolor{black}{Furthermore, we develop a data-free variant of ReForge that replaces activation statistics with task-vector Grams, eliminating the need for calibration examples.} Across extensive benchmarks, including up to 20-task merging in vision and 5-task merging in language, \textcolor{black}{ReForge} consistently outperforms all \textcolor{black}{evaluated} \textit{plug-and-play} anchor baselines (e.g., TA, WUDI-Merging, and TSV). \textcolor{black}{On 20-task ViT-B/32, ReForge improves the strongest evaluated baseline, ISO-CTS, from 77.6\% to 82.8\% in the data-assisted setting and to 81.5\% in the data-free setting. On eight-task ViT-L/14, the data-assisted variant achieves 95.1\% mean accuracy, compared with 95.8\% for the individual task experts.}
\textcolor{black}{Our source code is available at \url{https://anonymous.4open.science/r/reforge-6CE5}.}
\end{abstract}

\vspace{-6pt}
\section{Introduction}
Adapting foundation models to downstream tasks via fine-tuning has become standard practice, but maintaining a separate expert model for each task introduces substantial storage, deployment, and operational overhead. While multi-task learning offers a unified alternative~\citep{ruder2017overview}, it necessitates joint data access and incurs prohibitive \textcolor{black}{training} costs, making it often infeasible under data silos, privacy constraints, or limited computing budgets. Model merging~\citep{matena2022merging,ilharco2023task,yadav2023ties} offers a practical solution. By directly combining multiple expert models into a single architecture, it enables unified inference without revisiting the original training data or incurring the joint retraining costs. This paradigm has become increasingly relevant with the flourishing ecosystem of open-source models on platforms such as Hugging Face~\citep{huggingface_hub}, which offers an abundant supply of task-specific experts for \textcolor{black}{weight-space} integration.

A central challenge in model merging is how to effectively combine expert models when auxiliary data is limited or unavailable. Existing methods fall into two broad regimes based on their use of auxiliary data. \emph{Data-assisted} methods, such as Fisher Merging~\citep{matena2022merging} and RegMean~\citep{jin2023dataless}, rely on a small calibration set to estimate empirical statistics for merging. \emph{Data-free} methods, including Task Arithmetic (TA)~\citep{ilharco2023task}, TIES~\citep{yadav2023ties}, WUDI-Merging~\citep{cheng2025wudi}, TSV~\citep{gargiulo2025tsv}, and ISO-CTS~\citep{marczak2025iso}, avoid auxiliary data for estimating merged parameters, though some still use a held-out validation set for lightweight hyperparameter selection. \textcolor{black}{These methods produce strong merged models that capture useful information from multiple experts. However, these merged models are rarely used as priors to guide further refinement.}

To address \textcolor{black}{this limitation}, we propose \textcolor{black}{ReForge}, a \emph{plug-and-play} \textcolor{black}{bilevel} optimization framework that is built on two key ideas. \textcolor{black}{First, ReForge refines an existing merged model by treating it as an anchor and incorporating activation evidence from task experts. This yields an anchor-regularized regression problem with an efficient closed-form solution, using only unlabeled calibration examples.} Second, \textcolor{black}{ReForge} performs globally coordinated architecture-aware Bayesian optimization to \textcolor{black}{jointly select heterogeneous regularization strengths and assembly scales across different modules of the network.} Moreover, we \textcolor{black}{use task-vector Grams as weight-based surrogates for} activation statistics, enabling us to derive a data-free variant of \textcolor{black}{ReForge} while preserving the closed-form solution.

Empirically, \textcolor{black}{ReForge} is effective across both vision and language model merging. Across four backbones from two model families and seven \emph{plug-and-play} anchor models, \textcolor{black}{ReForge} consistently improves over every \textcolor{black}{evaluated} anchor in both data-assisted and data-free settings, with relative gains of up to \textcolor{black}{29.8}\% \textcolor{black}{over} weaker \textcolor{black}{merging baselines} while still improving the strongest ones. In particular, on the standard ViT-L/14 benchmark \textcolor{black}{for 8-task merging}, a single merged model reaches 95.1\%, closely matching the average performance of eight task-specific experts (95.8\%). Extensive ablation studies are conducted to further validate the main components and design choices of \textcolor{black}{ReForge}.

\raggedbottom
\vspace{-6pt}
\section{Related Work}
\paragraph{\textcolor{black}{Training-free model merging.}}
 Existing methods generally fall into two broad regimes based on their reliance on auxiliary data: data-assisted and data-free. \emph{Data-assisted} methods leverage a small calibration set or activation statistics to guide the merge. For example, Fisher Merging~\citep{matena2022merging} performs Fisher-weighted \textcolor{black}{averaging} of task-specific models, while RegMean~\citep{jin2023dataless} casts linear-layer merging as regression with a closed-form solution. \emph{Data-free} methods eliminate the need \textcolor{black}{for calibration examples when constructing merged weights} and rely on the weights of the expert models. In this regime, Task Arithmetic (TA)~\citep{ilharco2023task} \textcolor{black}{provides a baseline for combining task vectors.} Subsequent works mitigate interference among models through \textcolor{black}{selective parameter combination}, including TIES~\citep{yadav2023ties}, DARE~\citep{yu2024supermario}, PCB-Merging~\citep{du2024pcb}, and Localize-and-Stitch~\citep{he2025localize}. TSV~\citep{gargiulo2025tsv} and ISO-CTS~\citep{marczak2025iso} \textcolor{black}{leverage} structured subspaces to isolate task-specific features or align task-relevant subspaces. More recently, methods such as WUDI-Merging~\citep{cheng2025wudi} and DOGE~\citep{wei2025doge} optimize explicit data-free merging objectives. Our \textcolor{black}{ReForge} is complementary to these approaches: it leverages existing merged solutions as anchors, admits an efficient closed-form solution for module-level \textcolor{black}{refinement}, and exploits \textcolor{black}{bilevel} optimization for hyperparameter search across different module groups.

\paragraph{\textcolor{black}{Merging coefficient learning and hyperparameter search.}}
\textcolor{black}{Learning how to combine experts offers an alternative to fixed merging rules.
AdaMerging~\citep{yang2024adamerging} optimizes task-wise or layer-wise coefficients through entropy minimization on unlabeled test data, while aTLAS~\citep{zhang2024atlas} learns block-wise task-vector scalings under supervised or unsupervised objectives.
Alongside these gradient-based formulations, Bayesian optimization enables coefficient search driven directly by evaluation metrics.
SIP-BMM~\citep{chen2026sipbmm} guides layer-wise coefficient search with structural importance information to construct a capability}--\textcolor{black}{efficiency Pareto set, whereas DF-Merge~\citep{lee2025dynamic} couples model-level scaling with Fisher information recomputed at each candidate.
ReForge instead uses Bayesian optimization to refine an existing merged anchor. It jointly selects heterogeneous regularization strengths and assembly scales in an architecture-aware manner, using validation performance without backpropagation through the merged network.}

\vspace{-6pt}
\section{Problem Definition}
\paragraph{Notation.} Let $\theta_{\mathrm{pre}} \in \mathbb{R}^d$ denote the parameters of a pretrained base model, and  $\{\theta^{(t)}\}_{t=1}^T$ represent the parameters of $T$ distinct models fine-tuned from $\theta_{\mathrm{pre}}$ on different downstream tasks. All models share the same architecture and operate within the same parameter space $\mathbb{R}^d$.

\paragraph{Task Vectors.} Following the framework of Task Arithmetic (TA)~\citep{ilharco2023task}, we use task vectors to represent task-specific parameter offsets. Specifically, the task vector for the $t$-th model is defined as $\tau^{(t)} = \theta^{(t)} - \theta_{\mathrm{pre}}$, for $t \in \{1, \cdots, T\}$.

\textbf{Objective.}
Given a pretrained model $\theta_{\mathrm{pre}}$ and task vectors $\{\tau^{(t)}\}_{t=1}^T$, our goal is to design an aggregation strategy $\mathcal{A}$ that combines $T$ task vectors into a single merged model, parameterized by:
\begin{equation}
\theta_{\mathrm{merged}} = \theta_{\mathrm{pre}} + \mathcal{A}\big(\tau^{(1)}, \cdots, \tau^{(T)}\big),
\label{eq:aggregation}
\end{equation}
such that the merged model preserves the task-specific capabilities of all the fine-tuned models as much as possible.  

\vspace{-6pt}
\section{Methodology}
\label{sec:method}

\vspace{-4pt}
\subsection{Model Merging as Bayesian Linear Regression}
\label{sec:linear_regression}
\paragraph{Module-wise Decomposition.} 
While Eq.~\eqref{eq:aggregation} defines the merging objective at the full parameter space $\mathbb{R}^d$, deep neural networks are practically composed of multiple distinct modules. For computational tractability, we decompose the full parameter space $\mathbb{R}^d$ into module-wise parameter partitions, indexed by $m\textcolor{black}{\in}\{1, \cdots, M\}$. Following the approach of WUDI-Merging~\citep{cheng2025wudi}, our aggregation strategy focuses specifically on 2D weight matrices. \textcolor{black}{Below, we focus on a single module and omit its index for brevity.}

\textcolor{black}{Let $\mathbf W_{\mathrm{pre}},\mathbf W^{(t)}\in\mathbb R^{d_{\mathrm{out}}\times d_{\mathrm{in}}}$ denote the pretrained and expert weights, respectively. The corresponding expert task vector is $\mathbf U^{(t)}=\mathbf W^{(t)}-\mathbf W_{\mathrm{pre}}\in\mathbb R^{d_{\mathrm{out}}\times d_{\mathrm{in}}}$. The merged module weights are expressed in terms of a merged task vector $\mathbf U$ as}
\begin{equation}
\mathbf W_{\mathrm{merged}}=\mathbf W_{\mathrm{pre}}+s\mathbf U,
\label{eq:merged_w}
\end{equation}
where $s$ is \textcolor{black}{selected} on \textcolor{black}{the} validation set.

\vspace{-7pt}
\paragraph{Activation-based Regression.}
We frame the estimation of the merged module-wise task vector $\mathbf{U}$ as an activation-based regression  problem. In the data-assisted setting, we collect $N$ representative activations for the $t$-th task, denoted as $\{\mathbf{x}_n^{(t)}\}_{n=1}^N$, where $\mathbf{x}_n^{(t)} \in \mathbb{R}^{d_{\text{in}}}$. They are collected by passing the task-specific unlabeled calibration data through the corresponding fine-tuned model. To preserve the task-specific capabilities, we aim to align the residual output induced by the merged task vector $\mathbf{U}$ with that induced by the original \textcolor{black}{expert task vectors} $\{\mathbf{U}^{(t)}\}_{t=1}^T$. Specifically, we define the residual output as $\mathbf{y}_n^{(t)} = \mathbf{U}^{(t)}\mathbf{x}_n^{(t)} = (\mathbf{W}^{(t)} - \mathbf{W}_{\mathrm{pre}})\mathbf{x}_n^{(t)}$. \textcolor{black}{Before concatenation, we rescale each task's activation}--\textcolor{black}{residual pairs by $c_t^{-1/2}$, where $c_t=\sum_{n=1}^{N}\|\mathbf x_n^{(t)}\|_2^2>0$, giving each task a unit-trace activation Gram.} We \textcolor{black}{then} concatenate \textcolor{black}{the rescaled} vectors across tasks to \textcolor{black}{form} $\mathbf X$ and $\mathbf Y$:
\begingroup\color{black}\begin{equation}\label{eq:local_activation}
    \mathbf X=\left[c_1^{-1/2}\mathbf x_1^{(1)},\ldots,c_T^{-1/2}\mathbf x_N^{(T)}\right],\quad
    \mathbf Y=\left[c_1^{-1/2}\mathbf y_1^{(1)},\ldots,c_T^{-1/2}\mathbf y_N^{(T)}\right].
\end{equation}\endgroup
We assume a linear observation model where the residual outputs $\mathbf{Y}$ are generated by the merged module-wise task vector $\mathbf{U}$ acting on $\mathbf{X}$, corrupted by Gaussian noise $\mathbf{E}$:
\begin{equation}
    \mathbf{Y} = \mathbf{U}\mathbf{X} + \mathbf{E},
\end{equation}
where the noise matrix satisfies $\mathbf{E}_{:,j} \sim \mathcal{N}(0, \beta^{-1}\mathbf{I})$ independently for each column index $j$. Under this linear observation model, the likelihood function of observing $\mathbf{Y}$ given $\mathbf{U}$ can be expressed as:
\begin{equation}
    p(\mathbf{Y} \mid \mathbf{U}, \mathbf{X}, \beta) \propto \exp\left( -\frac{\beta}{2} \|\mathbf{Y} - \mathbf{U}\mathbf{X}\|_F^2 \right).
\end{equation}

\vspace{-7pt}
\paragraph{\textcolor{black}{Anchor Refinement with Activation Evidence.}} 
Purely data-driven estimation of $\mathbf{U}$ on limited calibration data is not only prone to overfitting but also leaves valuable inductive bias of anchor models (e.g., existing merging solutions) unused. To regularize the solution and inject such priors, we introduce a module-wise \emph{anchor} task vector, defined as $\mathbf{U}^{(0)} = \mathbf{W}_{\mathrm{anchor}} - \mathbf{W}_{\mathrm{pre}}$, where $\mathbf{W}_{\mathrm{anchor}}$ denotes the module weights obtained from an existing merging solution (e.g., TA~\citep{ilharco2023task}, TIES~\citep{yadav2023ties}, or TSV~\citep{gargiulo2025tsv}).  We then impose an element-wise independent Gaussian prior on $\mathbf{U}$ centered at anchor task vector $\mathbf{U}^{(0)}$:
\begin{equation}
    p(\mathbf{U} | \mathbf{U}^{(0)}, \alpha) \propto \exp\left( -\frac{\alpha}{2} \|\mathbf{U} - \mathbf{U}^{(0)}\|_F^2 \right),
\end{equation}
where $\alpha$ controls the precision of this anchor-induced prior. Given these formulations, the \emph{maximum a posteriori} (MAP) estimate of $\mathbf{U}$ is obtained by maximizing the posterior $p(\mathbf{U} | \mathbf{Y}, \mathbf{X}) \propto p(\mathbf{Y} | \mathbf{U}, \mathbf{X}, \beta) \,p(\mathbf{U} | \mathbf{U}^{(0)}, \alpha)$. Taking the negative logarithm transforms posterior maximization into a minimization problem. Let the regularization weight \textcolor{black}{$\lambda\!=\!\frac{\alpha}{\beta}>0$}, the MAP objective simplifies to:
\begin{equation}
\begin{aligned}
\mathbf{U}_{\text{MAP}} = \arg\min_{\mathbf{U}} \left( \|\mathbf{Y} - \mathbf{U}\mathbf{X}\|_F^2 + \lambda \|\mathbf{U} - \mathbf{U}^{(0)}\|_F^2 \right),
\label{eq:map_estimation}
\end{aligned}
\end{equation}
where $\lambda$ balances the empirical fit on the task-specific activations (the first term) with the prior knowledge encapsulated by the anchor model (the second term). \textcolor{black}{Setting the derivative of the objective with respect to $\mathbf U$ to zero yields the closed-form solution:}
\begin{equation}
    \mathbf{U}_{\text{MAP}} = \left( \mathbf{Y}\mathbf{X}^\top + \lambda \mathbf{U}^{(0)} \right) \left( \mathbf{X}\mathbf{X}^\top + \lambda \mathbf{I} \right)^{-1},
    \label{eq:closed_form}
\end{equation}
where $\mathbf{I} \in \mathbb{R}^{d_{\mathrm{in}} \times d_{\mathrm{in}}}$ is the identity matrix. Given a \textcolor{black}{$\lambda>0$}, this closed-form solution allows us to compute the optimal merged task vectors for all $M$ modules efficiently. \textcolor{black}{For $\lambda>0$, the MAP solution is unique and equivalent to anchor-centered ridge regression. The Gaussian posterior also provides a covariance matrix shared by all rows of $\mathbf U$:}
\[
\mathbf\Sigma_{\mathrm{row}}=\beta^{-1}(\mathbf X\mathbf X^\top+\lambda\mathbf I)^{-1}.
\]
\textcolor{black}{We can sample module weights from this posterior and average predictions from the resulting models to improve uncertainty calibration (Section~\ref{sec:posterior_calibration}).}

\input{figures/figure1}
\raggedbottom
\vspace{-4pt}
\subsection{\textcolor{black}{Joint Hyperparameter Selection via Bayesian Optimization}}
\label{sec:bo}
\textcolor{black}{Eq.~\eqref{eq:closed_form} solves the regression problem for each of the $M$ modules independently. However,} neural network modules are highly coupled, \textcolor{black}{so their} regularization strengths \textcolor{black}{and assembly scales must be selected jointly according} to the end-to-end performance of the full network. \textcolor{black}{We use a held-out validation set to select these hyperparameters. Let $\mathbf h=(\boldsymbol\lambda,\mathbf s)\in\Omega$ denote a candidate configuration, with $\lambda_m>0$ and $s_m$ assigned to module $m$. We evaluate each candidate using the average validation accuracy across $T$ tasks:}
\begingroup\color{black}\begin{equation}
\mathrm{Score}(\mathbf h,\mathcal D_{\mathrm{val}})
=\frac{1}{T}\sum_{t=1}^{T}
\mathrm{Acc}_t\!\left(\theta_{\mathrm{merged}}(\mathbf h),\mathcal D_{\mathrm{val}}^{(t)}\right).
\label{eq:score_def}
\end{equation}\endgroup
\textcolor{black}{Here, $\theta_{\mathrm{merged}}(\mathbf h)$ is constructed from the refined modules as}
\begingroup\color{black}\begin{equation}
\mathbf W_m(\mathbf h)=\mathbf W_{\mathrm{pre},m}+s_m\mathbf U_m^\star(\lambda_m).
\label{eq:outer_assembly}
\end{equation}\endgroup
\textcolor{black}{This gives the \textcolor{black}{bilevel} optimization problem}
\begingroup\color{black}\begin{gather}
\mathbf h^\star\in\arg\max_{\mathbf h\in\Omega}\mathrm{Score}(\mathbf h,\mathcal D_{\mathrm{val}}),
\label{eq:bilevel_outer}\\
\text{subject to}\quad
\mathbf U_m^\star(\lambda_m)=\arg\min_{\mathbf U_m}
\left[\|\mathbf Y_m-\mathbf U_m\mathbf X_m\|_F^2
+\lambda_m\|\mathbf U_m-\mathbf U_m^{(0)}\|_F^2\right],
\label{eq:bilevel_inner}
\end{gather}\endgroup
where $\Omega$ denotes the search space of the regularization strengths and assembly scales in $\mathbf h$.
The inner \textcolor{black}{problem is the module-wise MAP estimation problem} in Eq.~\textcolor{black}{\eqref{eq:map_estimation}, with the} closed-form solution \textcolor{black}{in} Eq.~\eqref{eq:closed_form}. \textcolor{black}{The} remaining task is to optimize the outer objective in Eq.~\eqref{eq:bilevel_outer}. We treat this objective as a black-box function of \textcolor{black}{$\mathbf h$:} each evaluation requires \textcolor{black}{assembling} a full merged model and measuring its validation performance.

\textcolor{black}{For efficient outer-loop optimization, we adopt Bayesian optimization with a Gaussian process ($\mathcal{GP}$) surrogate}~\citep{frazier2018bayesian}. \textcolor{black}{The outer loop proposes hyperparameter combinations by maximizing Expected Improvement (EI), while the inner loop computes the closed-form module updates used to construct each candidate model. The resulting validation score updates the surrogate and guides subsequent proposals. Figure~\ref{fig:bmm_arch} illustrates this interaction, and Algorithm~\ref{alg:overall} details the procedure.}

\textcolor{black}{Searching separate hyperparameters for all $M$ modules creates a high-dimensional search space} (\textcolor{black}{$M=96$ for ViT-L}/\textcolor{black}{14 and $M=224$ for Llama-3.1}-\textcolor{black}{8B})\textcolor{black}{.} To \textcolor{black}{make the} search \textcolor{black}{tractable while retaining module-specific control}, we \textcolor{black}{use} block-wise parameter tying. \textcolor{black}{We} partition consecutive Transformer layers into $B$ sequential blocks. Within each block, modules \textcolor{black}{with the same} functional \textcolor{black}{role share a regularization strength, yielding} four groups: attention-in (Q/K/V), attention-out, MLP-in, and MLP-out. \textcolor{black}{Together with one block-specific assembly scale $s$} (\textcolor{black}{Eq.~\eqref{eq:merged_w}})\textcolor{black}{, this gives five search variables per block and $5B$ variables in total.}

\input{sections/overall_algorithm}

\vspace{-4pt}
\subsection{From Data-Assisted to Data-Free}
\label{sec:data_free}
\begingroup
\vspace{-7pt}
\paragraph{\textcolor{black}{Task-vector Gram surrogate.}}
\textcolor{black}{The closed-form solution in Eq.~\eqref{eq:closed_form} requires activation statistics estimated from unlabeled calibration data. When such data are unavailable, a weight-based surrogate is needed. ACE-Merging~\citep[Theorem~1]{xu2026ace} establishes a proportional relationship between the expected gradient Gram and input second-order statistics under its stated assumptions. Motivated by this connection, we replace each normalized activation Gram with}
\begingroup\color{black}\begin{equation}
\mathbf B_t=\frac{(\mathbf U^{(t)})^\top\mathbf U^{(t)}}{\|\mathbf U^{(t)}\|_F^2},\qquad \mathbf U^{(t)}\ne\mathbf 0.
\label{eq:normalized_df_gram}
\end{equation}\endgroup
\textcolor{black}{The denominator normalizes each task-vector Gram to unit trace, removing differences in overall scale across tasks. We use the normalized Gram $\mathbf B_t$ to approximate activation statistics without calibration data.}

\vspace{-7pt}
\paragraph{\textcolor{black}{Data-free refinement.}}
\textcolor{black}{Replacing each normalized activation Gram in Eq.~\eqref{eq:closed_form} with $\mathbf B_t$, and using the relation $\mathbf Y^{(t)}=\mathbf U^{(t)}\mathbf X^{(t)}$, yields the data-free closed-form update:}
\begingroup\color{black}\begin{equation}
\mathbf U_{\mathrm{DF}}^*(\tilde\lambda)
=\left(\sum_{t=1}^T\mathbf U^{(t)}\mathbf B_t+\tilde\lambda\mathbf U^{(0)}\right)
\left(\sum_{t=1}^T\mathbf B_t+\tilde\lambda\mathbf I\right)^{-1}.
\label{eq:data_free_closed_form}
\end{equation}\endgroup
\textcolor{black}{Here, $\tilde\lambda>0$ controls anchor regularization in the data-free setting. This update requires only the expert task vectors and the anchor. The task-vector fitting objective underlying Eq.~\eqref{eq:data_free_closed_form} coincides with the interference objective of WUDI-Merging~\citep[Eq.~21]{cheng2025wudi}. ReForge-DF regularizes the solution toward an externally constructed anchor task vector $\mathbf U^{(0)}$ and admits a closed-form solution.}

\vspace{-7pt}
\paragraph{\textcolor{black}{Empirical evidence.}}
\textcolor{black}{We examine whether task-vector Grams capture useful activation structure and support local refinement} (\textcolor{black}{Figure~\ref{fig:gram_diagnostics}})\textcolor{black}{. Compared with random PSD controls, task-vector Grams show greater overlap with the dominant activation subspace} (\textcolor{black}{Figure~\ref{fig:gram_diagnostics}}(\textcolor{black}{a})) \textcolor{black}{and yield module-level outputs closer to those obtained with activation Grams} (\textcolor{black}{Figure~\ref{fig:gram_diagnostics}}(\textcolor{black}{b}))\textcolor{black}{. These findings support their use as surrogates for activation statistics in refinement. Experimental details and additional results are provided in Appendix~\ref{sec:appendix_a}.}

\vspace{-7pt}
\paragraph{\textcolor{black}{Refinement with few calibration examples.}}
\textcolor{black}{When a small calibration set is available, we could supplement the weight-based surrogate with activation statistics. Let $\mathbf X_{t,n}$ collect the module's input activations from $n$ unlabeled calibration examples for task $t$. We define the trace-normalized activation Gram as}
\begingroup\color{black}\[
\widehat{\mathbf G}_{t,n}^{\mathrm{act}}
=\frac{\mathbf X_{t,n}\mathbf X_{t,n}^{\top}}
{\operatorname{tr}(\mathbf X_{t,n}\mathbf X_{t,n}^{\top})},
\]\endgroup
\textcolor{black}{and combine it with the task-vector Gram $\mathbf B_t$:}
\begingroup\color{black}\begin{equation}
\mathbf G_{t,n}^{\mathrm{mix}}(\rho)
=(1-\rho)\mathbf B_t+\rho\widehat{\mathbf G}_{t,n}^{\mathrm{act}},
\qquad 0\le\rho\le1.
\label{eq:fewshot_mix}
\end{equation}\endgroup
\textcolor{black}{Here, $\rho$ controls their relative contributions: $\rho=0$ uses only task-vector statistics, whereas $\rho=1$ uses only activation statistics. Replacing $\mathbf B_t$ in Eq.~\eqref{eq:data_free_closed_form} with this mixed Gram gives the corresponding refinement update. Experimental results are reported in Section~\ref{sec:additional_analyses}.}

\par\endgroup

\flushbottom
\vspace{-6pt}
\section{Experiments}
\label{sec:experiments}
\textcolor{black}{We evaluate ReForge-DA and ReForge-DF, the data-assisted and data-free variants of ReForge, on vision and language benchmarks with four backbones and seven anchor choices. Experiments run on eight NVIDIA RTX 6000 Ada 48GB GPUs.}

\vspace{-4pt}
\subsection{Experimental Settings}
\label{sec:experiments_setting}
\vspace{-7pt}
\paragraph{Datasets and Models.}
For \textit{vision} tasks, we follow the scalability evaluation settings of \textcolor{black}{\citet{gargiulo2025tsv} and \citet{marczak2025iso}}, covering 8-, 14-, and
20-task merging scenarios with ViT-B/32 and ViT-L/14 backbones. For
\textit{language} tasks, following \textcolor{black}{\citet{he2025mergebench}}, we evaluate \textcolor{black}{ReForge} for
5-task merging based on Llama-3.2-3B and
Llama-3.1-8B. \textcolor{black}{We use separate validation and test sets for hyperparameter selection and final evaluation, respectively.} Full benchmark descriptions, metrics, and asset/license info are provided in Appendix~\ref{app:exp-protocol}.

\vspace{-7pt}
\paragraph{Baselines.} \textcolor{black}{We use six merging methods as both baselines and sources of anchor models: TA~\citep{ilharco2023task}, TIES~\citep{yadav2023ties}, RegMean~\citep{jin2023dataless}, TSV~\citep{gargiulo2025tsv}, WUDI-Merging~\citep{cheng2025wudi}, and ISO-CTS~\citep{marczak2025iso}. We additionally use the pretrained backbone as an anchor, corresponding to a zero-centered prior with $\mathbf U^{(0)}=\mathbf 0$, and report individual expert performance for reference.}

\vspace{-7pt}
\paragraph{\textcolor{black}{Experimental details.}} Following RegMean~\citep{jin2023dataless}, \textcolor{black}{ReForge-DA} collects local activation data $\mathcal{D}_{\text{act}} = \{\mathbf{X}_m, \mathbf{Y}_m\}_{m=1}^M$ using 128 and 1,000 samples per task for ViT and Llama, respectively. We partition the networks into $B=3$ (ViT) and $B=1$ (Llama) sequential blocks. Each block has 5 hyperparameters to tune: one scale $s \in [1.0, 1.3]$ (Eq.~\ref{eq:merged_w}) and four group-wise regularization strengths (attention-in/out and MLP-in/out). \textcolor{black}{We search the regularization strengths in logarithmic coordinates over $[10^{-4},1]$ for ViT and $[10^{-3},100]$ for Llama.} \textcolor{black}{We use $K=200$ BO trials for ViT and $K=100$ for Llama to search the 15-dimensional and five-dimensional hyperparameter spaces, respectively.} Additional complexity analysis and wall-clock runtime breakdowns are provided in Appendix~\ref{app:runtime}.

\input{tables/main_vision.tex}
\input{tables/main_language.tex}

\vspace{-4pt}
\subsection{Main Results}
\vspace{-7pt}
\paragraph{Vision Benchmarks (ViT).}
Table~\ref{tab:merged_performance_comparison} shows that both ReForge-DA and ReForge-DF improve task accuracy over every evaluated anchor across the tested model sizes and task counts. ReForge-DF achieves average relative improvements of 21.1\% over TA and 16.2\% over TIES, demonstrating that substantial gains are possible without calibration data. The improvements extend to strong anchors such as WUDI-Merging and ISO-CTS. Although ISO-CTS already performs well on ViT merging benchmarks, both ReForge variants consistently improve its performance, including when merging 20 tasks. On ViT-B/32, ReForge-DA and ReForge-DF raise its accuracy from 77.6\% to 82.8\% and 81.5\%, respectively. Similar gains are observed on ViT-L/14.

ReForge also narrows the gap between merged models and individual task-specific experts. On the ViT-L/14 eight-task benchmark with the ISO-CTS anchor, ReForge-DA reaches 95.1\% and ReForge-DF reaches 95.0\%, compared with 95.8\% for the individual experts. Standard deviations across five seeds and per-task results are provided in Appendices~\ref{app:detailed-aggregate} and~\ref{app:vit-pertask}, respectively.

\vspace{-7pt}
\paragraph{Language Benchmarks (Llama).}
Table~\ref{tab:nlp_main} shows that both ReForge variants improve performance over every evaluated anchor on the five-task language benchmarks. The gains extend from weaker anchors to strong merging baselines. On Llama-3.2-3B, ReForge-DF achieves a relative improvement of 49.3\% over the pretrained anchor and raises the TSV score from 0.471 to 0.486, achieving the best result among the evaluated 3B merging methods. ReForge-DF also often outperforms ReForge-DA on this backbone.

On Llama-3.1-8B, ReForge-DF improves the TSV score from 0.557 to 0.573, while ReForge-DA with TIES achieves the best result among the evaluated 8B merging methods at 0.579. The relative performance of the two variants therefore depends on the backbone and anchor. Category-level results are provided in Appendix~\ref{app:vit-pertask}.


\input{tables/ablation_BO}

\vspace{-4pt}
\subsection{Ablation Study}
\vspace{-7pt}
To investigate the effectiveness of BO-based global coordination, Table~\ref{tab:bo_ablation} compares \textcolor{black}{ReForge} (\textcolor{black}{with} \textit{BO}) against two hyperparameter-tuning baselines: \textit{shared-}$\lambda$ and \textit{random search}. While \textit{shared-}$\lambda$ assigns a single regularization strength $\lambda$ to all four module groups and optimizes it via an extensive grid search, \textit{random search} relaxes this constraint by allowing module-specific hyperparameters, but tunes them using 200 trials of random search.

\textcolor{black}{Shared-$\lambda$ improves all evaluated anchors, showing that refinement remains effective with a shared regularization strength.} Relaxing the shared constraint to module-specific hyperparameters with \textit{random search} further improves the performance, validating the importance of heterogeneous regularization across different module groups. Finally, \textcolor{black}{ReForge} (\textcolor{black}{with} \textit{BO}) achieves the best overall results by effectively coordinating the module-specific hyperparameters with guided search. Its advantage is most pronounced in the challenging 20-task data-free setting. On the TSV anchor, \textcolor{black}{ReForge} (\textcolor{black}{with} \textit{BO}) reaches an accuracy of $79.7\%$, compared with $79.0\%$ for \textit{random search} and $77.9\%$ for \textit{shared-$\lambda$} tuning. This suggests that BO becomes particularly useful when the search space is more complex and the validation signal needs to balance task interference effectively in large-scale merging settings.

We further examine whether ReForge's gains can be achieved solely by applying heterogeneous scaling coefficients to the anchor task vectors across module groups and blocks, without regression refinement. Both ReForge-DA and ReForge-DF outperform anchor rescaling across six ViT-B/32 task--anchor configurations with the same nominal search dimension and validation-evaluation budget (Appendix~\ref{app:anchor-rescaling}). We also compare task-vector Grams with random positive semidefinite (PSD) and isotropic matrices under a common evaluation and search protocol. Across the six configurations, task-vector Grams improve ReForge-DF's mean test accuracy by 0.81--1.40 percentage points over random PSD matrices and by 0.81--1.29 percentage points over isotropic matrices (Appendix~\ref{app:gram_controls}).

\begin{figure}[t]
\centering
\captionsetup{justification=raggedright,singlelinecheck=false}
\begin{minipage}[t]{0.49\textwidth}
\centering
\includegraphics[width=\linewidth]{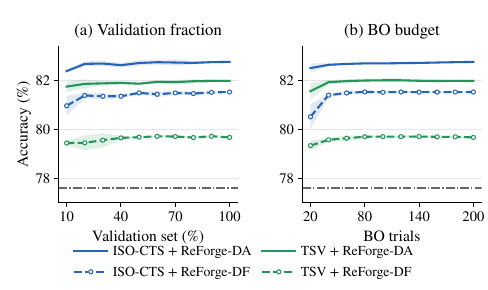}
\caption{\textbf{\textcolor{black}{ReForge ablations on ViT-B}/\textcolor{black}{32, 20 tasks.}} \textcolor{black}{Test accuracy versus validation fraction} (\textcolor{black}{left}) \textcolor{black}{and BO trials $K$} (\textcolor{black}{right})\textcolor{black}{; mean $\pm$ std over five seeds.} The gray dot-dashed line \textcolor{black}{is} the ISO-CTS baseline.}
\label{fig:bmm-ablation}
\end{minipage}\hfill
\begin{minipage}[t]{0.49\textwidth}
\centering
\includegraphics[width=\linewidth]{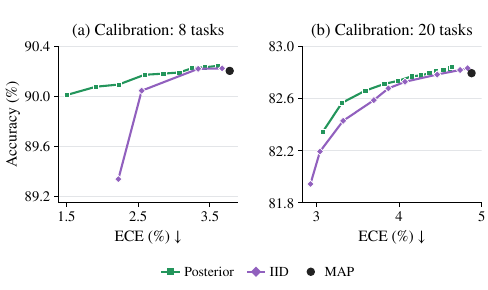}
\caption{\textbf{\textcolor{black}{Posterior calibration} on ViT-B/32.} \textcolor{black}{Test-set accuracy}--\textcolor{black}{ECE Pareto frontiers for} (\textcolor{black}{a}) \textcolor{black}{8 and} (\textcolor{black}{b}) \textcolor{black}{20 tasks, using ISO-CTS. Posterior and IID sampling each average ten predictions; black points mark MAP.}}
\label{fig:covcal_vitb32}
\end{minipage}
\vspace{-4pt}
\end{figure}

\subsection{\textcolor{black}{Additional Analyses}}
\label{sec:additional_analyses}
\paragraph{Sensitivity to the Fraction of Validation Set.}
Figure~\ref{fig:bmm-ablation} (left) reports the final test performance as the fraction of validation set used for BO evolves. Across both TSV and ISO-CTS anchors, in data-assisted and data-free settings, \textcolor{black}{ReForge} consistently outperforms the strongest anchor baseline (ISO-CTS) even when a small fraction of the validation set is used. The performance curves are relatively stable from $10\%$ to $100\%$ of the validation set, indicating that the outer-loop optimization is not overly sensitive to the validation-set size. This suggests that \textcolor{black}{ReForge} can use a small validation subset to reduce the optimization overhead while preserving most of the gains from global hyperparameter search.

\vspace{-6pt}
\paragraph{Sensitivity to BO Budget $K$.}
Figure~\ref{fig:bmm-ablation} (right) reports the evolution of final test performance as the number of BO trials ($K$) increases. \textcolor{black}{ReForge} consistently outperforms the strongest anchor baseline (ISO-CTS) under a small BO budget ($K$). The performance curves reach a plateau after roughly \textcolor{black}{$40$}--\textcolor{black}{$60$} trials, while increasing the budget further yields only a marginal return. These results indicate that BO is efficient in global hyperparameter search, reaching stable results with a modest budget. Additional \textcolor{black}{runtime}--\textcolor{black}{performance} comparisons against other methods are reported in Appendix~\ref{app:runtime-tradeoff}.

\vspace{-6pt}
\paragraph{\textcolor{black}{Few-shot calibration and Gram mixing.}}
\label{sec:fewshot}
\textcolor{black}{On ViT-B}/\textcolor{black}{32, Gram mixing matches or improves upon activation-only mean accuracy in all nine task-count}--\textcolor{black}{anchor configurations with either one or five calibration examples per task} (\textcolor{black}{Table~\ref{tab:gram_ablation}})\textcolor{black}{. For ISO-CTS on 20 tasks, one-shot mixing raises accuracy from $80.5\%$ to $82.8\%$, matching the 128-sample ReForge-DA result. Across the nine configurations, one}- \textcolor{black}{and five-shot mixing match or exceed the 128-sample ReForge-DA mean in four and seven cases, respectively.}
\input{tables/fewshot}

\vspace{-6pt}
\paragraph{\textcolor{black}{Posterior sampling and calibration.}}
\label{sec:posterior_calibration}
We evaluate whether posterior sampling improves the uncertainty calibration of ReForge-DA on the ViT-B/32 8- and 20-task benchmarks with the ISO-CTS anchor. Starting from the same MAP solution, we compare sampling from the posterior with adding IID Gaussian noise. For each method, we generate ten sampled models and average their predicted class probabilities. We vary the sampling parameters and plot the resulting accuracy--calibration Pareto frontiers on the test sets in Figure~\ref{fig:covcal_vitb32}. Posterior sampling achieves a better trade-off between accuracy and calibration. Appendix~\ref{app:posterior-sampling} provides the experimental details.

\vspace{-6pt}
\section{Conclusion}
This paper introduces \textbf{\textcolor{black}{ReForge}}, a \emph{plug-and-play} framework for \textcolor{black}{refining existing merged models}. \textcolor{black}{ReForge incorporates expert-derived information through module-wise Bayesian linear regression with an anchor-centered prior and a closed-form posterior. Bayesian optimization jointly selects heterogeneous regularization strengths and assembly scales according to validation performance.}
We also derive a data-free variant of \textcolor{black}{ReForge using task-vector Grams as surrogates for} activation statistics, while preserving the closed-form solution. Across vision and language benchmarks, \textcolor{black}{ReForge} consistently improves diverse anchor baselines, scales to challenging 20-task settings, and closely approaches the performance of individual fine-tuned experts on high-capacity ViT architectures. Our results show that combining informative anchors with global hyperparameter coordination provides a practical path towards \textcolor{black}{refining merged models} when calibration data \textcolor{black}{are} limited or unavailable.

\input{sections/ai_use_statement}
\input{sections/revision_bib_colors}
\bibliographystyle{iclr2027_conference}  
\bibliography{ref}
\clearpage
\setlength{\textfloatsep}{10pt plus 2pt minus 2pt}
\setlength{\floatsep}{8pt plus 2pt minus 2pt}
\setlength{\intextsep}{8pt plus 2pt minus 2pt}
\appendix
\raggedbottom
\input{sections/gram_analysis}

\input{sections/appendix}

\end{document}

%% file: figures/figure1.tex
\begin{figure}[t]
    \centering
    \begin{minipage}[b]{0.48\textwidth}
        \centering
        \resizebox{\linewidth}{!}{
        \begin{tikzpicture}[>=Stealth, thick]
        \tikzstyle{obs} = [circle, draw, fill=gray!20, minimum size=0.9cm, inner sep=0pt, font=\normalsize]
        \tikzstyle{latent} = [circle, draw, fill=white, minimum size=0.9cm, inner sep=0pt, font=\normalsize]
        \tikzstyle{plate} = [draw, rectangle, rounded corners=6pt, inner xsep=10pt, inner ysep=14pt]
        \tikzstyle{alg} = [rectangle, draw=blue!60, fill=blue!5, rounded corners=4pt, minimum height=0.8cm, minimum width=2.1cm, font=\small, align=center, text=blue!80!black]
        
        \tikzstyle{terminal} = [circle, draw, fill=white, inner sep=1.5pt]
        \tikzstyle{pivot} = [circle, fill=black, inner sep=1.5pt]

        \node[latent] (lambda) at (0, 1.5)    {\textcolor{black}{$\mathbf h$}};
        \node[latent] (W)      at (0, 0)      {$\mathbf{U}_m$};
        \node[obs]    (W0)     at (-2.2, 0)   {$\mathbf{U}_m^{(0)}$};
        \node[obs]    (Dval)   at (3.4, 0)    {$\mathcal{D}_{\mathrm{val}}$};

        \node[obs]    (Y)      at (-1.2, -2.0) {$\mathbf{Y}_m$};
        \node[obs]    (X)      at (-2.8, -2.0) {$\mathbf{X}_m$};
        \draw[->] (X) -- (Y); 

        \node[obs]    (Ut)     at (1.2, -2.0)  {\small $\{\mathbf{U}_m^{(t)}\}$};

        \node[plate, fit=(W0) (W) (X) (Y) (Ut)] (plate) {};
        
        \node[anchor=south west, xshift=45pt, yshift=-1pt, font=\normalsize] at (plate.south west) {$m\textcolor{black}{\in}\{1,\cdots, M\}$};

        \draw[->] (lambda) -- (W);
        \draw[->] (W0)     -- (W);
        \draw[->] (W)      -- (Dval); 

        
        \node[pivot] (switch_pivot) at (0, -0.7) {};
        \draw[-] (W) -- (switch_pivot);

        \node[terminal] (term_L) at (-0.5, -1.1) {};
        \node[terminal] (term_R) at (0.5, -1.1) {};

        \draw[->] (term_L) -- node[left, font=\scriptsize, xshift=-2pt, yshift=2pt] {data-assisted} (Y);
        \draw[->] (term_R) -- node[right, font=\scriptsize, xshift=2pt, yshift=2pt] {data-free} (Ut);

        \draw[-, ultra thick, black!70] (switch_pivot) -- (term_L);

        \draw[<->, densely dashed, gray, bend right=40] (term_L) to (term_R);

        \node[alg] (BO) at (3.4, 1.5) {\textcolor{black}{Validation score} \\ \textcolor{black}{and BO update}};
        \draw[->, densely dashed, blue!80!black] (Dval) -- (BO);
        \draw[->, densely dashed, blue!80!black] (BO)   -- (lambda);

        \end{tikzpicture}
        }
         
    \end{minipage}
    \hfill
    \begin{minipage}[b]{0.48\textwidth}
        \centering
        \includegraphics[width=\linewidth]{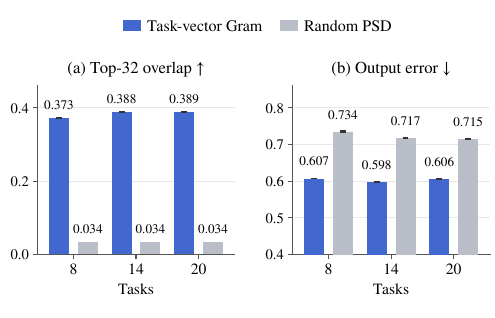}
         
    \end{minipage}
     
    \begin{minipage}[t]{0.48\textwidth}
    \caption{\textbf{\textcolor{black}{Inner MAP and outer BO.}} \textcolor{black}{ReForge computes closed-form updates from activations or task-vector Gram surrogates. Models assembled with selected scales are evaluated on held-out validation data, guiding joint BO over $\mathbf h$.}}
        \label{fig:bmm_arch}
    \end{minipage}
    \hfill
    \begin{minipage}[t]{0.48\textwidth}
        \caption{\textcolor{black}{\textbf{Task-vector Gram diagnostics.} ViT-B/32 with 8/14/20 tasks. (a) Top-32 subspace overlap with activation Grams ($\uparrow$). (b) Relative output error against activation-based refinement ($\downarrow$; ISO-CTS anchor).}}
        \label{fig:gram_diagnostics}
    \end{minipage}
\vspace{-6pt}
\end{figure}

%% file: sections/overall_algorithm.tex
\begin{algorithm}[!t]
\caption{\textcolor{black}{ReForge-DA refinement}}
\label{alg:overall}
\begin{algorithmic}[1]
\Require \textcolor{black}{Anchor task vectors} $\mathcal U^{(0)}=\{\mathbf U_m^{(0)}\}_{m=1}^M$, data $\{\mathbf X_m,\mathbf Y_m\}_{m=1}^M$, \textcolor{black}{val} set $\mathcal D_{\mathrm{val}}$, \#Iters $K$
\Ensure \textcolor{black}{Merged model $\theta_{\mathrm{merged}}$}
\State Initialize evaluation history $\mathcal H\leftarrow\emptyset$ and $\mathcal{GP}$ surrogate $\mathcal M$
\For{\textcolor{black}{$k=1,2,\ldots,K$}}
    \State \textit{\# Outer Loop: Surrogate-Guided Proposal}
    \State \textcolor{black}{$(\boldsymbol\lambda^{(k)},\mathbf s^{(k)})\leftarrow\arg\max_{\boldsymbol\lambda,\mathbf s}\mathrm{EI}((\boldsymbol\lambda,\mathbf s);\mathcal M)$} \Comment{Maximize acquisition function}
    \State \textit{\# Inner Loop: Closed-form MAP solution}
    \For{each module \textcolor{black}{$m\in\{1,\ldots,M\}$}}
        \State \textcolor{black}{$\mathbf U_m^\star\leftarrow(\mathbf Y_m\mathbf X_m^\top+\lambda_m^{(k)}\mathbf U_m^{(0)})(\mathbf X_m\mathbf X_m^\top+\lambda_m^{(k)}\mathbf I)^{-1}$} \Comment{Solve inner MAP via Eq.~\eqref{eq:closed_form}}
        \State \textcolor{black}{$\mathbf W_m\leftarrow\mathbf W_{\mathrm{pre},m}+s_m^{(k)}\mathbf U_m^\star$}
    \EndFor
    \State Assemble full merged model \textcolor{black}{$\theta_{\mathrm{merged}}^{(k)}\leftarrow\{\mathbf W_m\}_{m=1}^M$}
    \State \textit{\# Outer Loop: Validation evaluation}
    \State \textcolor{black}{$f^{(k)}\leftarrow\mathrm{Score}((\boldsymbol\lambda^{(k)},\mathbf s^{(k)}),\mathcal D_{\mathrm{val}})$}
    \State \textcolor{black}{$\mathcal H\leftarrow\mathcal H\cup\{((\boldsymbol\lambda^{(k)},\mathbf s^{(k)}),f^{(k)})\}$}; Update surrogate $\mathcal M$
\EndFor
\State \textcolor{black}{$(\boldsymbol\lambda^\star,\mathbf s^\star)\leftarrow\arg\max_{((\boldsymbol\lambda,\mathbf s),f)\in\mathcal H}f$} \Comment{\textcolor{black}{Extract best observed configuration}}
\State \Return \textcolor{black}{$\theta_{\mathrm{merged}}(\boldsymbol\lambda^\star,\mathbf s^\star)$}
\end{algorithmic}
\end{algorithm}

%% file: tables/main_vision.tex
\begin{table*}[t]
\caption{Model merging on vision tasks with ViT architectures. The blue subscript values indicate the \textbf{relative percentage improvement} compared to the corresponding \textit{anchor}. (Indiv.: Individual)}
\label{tab:merged_performance_comparison}
\centering

\definecolor{DeepBlue}{HTML}{00509E}
\definecolor{LightCyan}{HTML}{E0FFFF} 

\newcommand{\up}[1]{\textsubscript{\textcolor{blue}{+#1}}}
\newcommand{\upavg}[1]{{\footnotesize+#1}}

\setlength{\tabcolsep}{3.5pt}

\resizebox{\textwidth}{!}{
\begin{tabular}{l c l c *{7}{r@{}l}}
\toprule
\textbf{Model} & \textbf{Tasks} & \textbf{Setting} & \textbf{Indiv.} &
\multicolumn{2}{c}{\textbf{Pretrained}} & \multicolumn{2}{c}{\textbf{TA}} &
\multicolumn{2}{c}{\textbf{TIES}} & \multicolumn{2}{c}{\textbf{RegMean}} &
\multicolumn{2}{c}{\textbf{TSV}} & \multicolumn{2}{c}{\textbf{WUDI}} &
\multicolumn{2}{c}{\textbf{ISO-CTS}} \\
\midrule

\multirow{9}{*}{ViT-B/32}
 & \multirow{3}{*}{8}
 & \cellcolor{gray!10}anchor   & \cellcolor{gray!10}92.8 
 & \cellcolor{gray!10} & \cellcolor{gray!10}\llap{47.7}
 & \cellcolor{gray!10} & \cellcolor{gray!10}\llap{70.4}
 & \cellcolor{gray!10} & \cellcolor{gray!10}\llap{75.7}
 & \cellcolor{gray!10} & \cellcolor{gray!10}\llap{82.3}
 & \cellcolor{gray!10} & \cellcolor{gray!10}\llap{85.9}
 & \cellcolor{gray!10} & \cellcolor{gray!10}\llap{\textbf{87.0}}
 & \cellcolor{gray!10} & \cellcolor{gray!10}\llap{86.4} \\
 & & \textcolor{black}{ReForge-DA} & - 
 & 84.8 & \up{77.8} & 85.0 & \up{20.7} & 85.7 & \up{13.1} & 85.2 & \up{3.5} & 88.9 & \up{3.5} & 89.4 & \up{2.8} & \textbf{90.2} & \up{4.3} \\
 & & \textcolor{black}{ReForge-DF}     & - 
 & 87.9 & \up{84.4} & 87.0 & \up{23.6} & 87.0 & \up{14.9} & 87.9 & \up{6.8} & 87.8 & \up{2.2} & 87.6 & \up{0.7} & \textbf{89.0} & \up{3.0} \\
\cmidrule{2-18}
 & \multirow{3}{*}{14}
 & \cellcolor{gray!10}anchor   & \cellcolor{gray!10}90.9 
 & \cellcolor{gray!10} & \cellcolor{gray!10}\llap{56.9}
 & \cellcolor{gray!10} & \cellcolor{gray!10}\llap{65.2}
 & \cellcolor{gray!10} & \cellcolor{gray!10}\llap{68.2}
 & \cellcolor{gray!10} & \cellcolor{gray!10}\llap{76.6}
 & \cellcolor{gray!10} & \cellcolor{gray!10}\llap{79.9}
 & \cellcolor{gray!10} & \cellcolor{gray!10}\llap{80.5}
 & \cellcolor{gray!10} & \cellcolor{gray!10}\llap{\textbf{81.5}} \\
 & & \textcolor{black}{ReForge-DA} & - 
 & 78.9 & \up{38.8} & 79.2 & \up{21.5} & 79.7 & \up{16.8} & 79.1 & \up{3.2} & 84.3 & \up{5.5} & 84.6 & \up{5.1} & \textbf{85.5} & \up{4.9} \\
 & & \textcolor{black}{ReForge-DF}     & - 
 & 82.7 & \up{45.4} & 82.4 & \up{26.4} & 82.3 & \up{20.6} & 82.9 & \up{8.2} & 82.8 & \up{3.7} & 81.7 & \up{1.5} & \textbf{84.3} & \up{3.5} \\
\cmidrule{2-18}
 & \multirow{3}{*}{20}
 & \cellcolor{gray!10}anchor   & \cellcolor{gray!10}91.3 
 & \cellcolor{gray!10} & \cellcolor{gray!10}\llap{55.7}
 & \cellcolor{gray!10} & \cellcolor{gray!10}\llap{60.4}
 & \cellcolor{gray!10} & \cellcolor{gray!10}\llap{64.0}
 & \cellcolor{gray!10} & \cellcolor{gray!10}\llap{72.2}
 & \cellcolor{gray!10} & \cellcolor{gray!10}\llap{76.9}
 & \cellcolor{gray!10} & \cellcolor{gray!10}\llap{76.1}
 & \cellcolor{gray!10} & \cellcolor{gray!10}\llap{\textbf{77.6}} \\
 & & \textcolor{black}{ReForge-DA} & - 
 & 75.6 & \up{35.7} & 75.5 & \up{25.0} & 76.3 & \up{19.2} & 75.9 & \up{5.0} & 82.0 & \up{6.7} & 82.2 & \up{8.0} & \textbf{82.8} & \up{6.6} \\
 & & \textcolor{black}{ReForge-DF}     & - 
 & 80.0 & \up{43.5} & 78.4 & \up{29.8} & 79.2 & \up{23.7} & 79.7 & \up{10.4} & 79.7 & \up{3.7} & 78.0 & \up{2.6} & \textbf{81.5} & \up{5.0} \\

\midrule
\multirow{9}{*}{ViT-L/14}
 & \multirow{3}{*}{8}
 & \cellcolor{gray!10}anchor   & \cellcolor{gray!10}95.8 
 & \cellcolor{gray!10} & \cellcolor{gray!10}\llap{65.1}
 & \cellcolor{gray!10} & \cellcolor{gray!10}\llap{84.8}
 & \cellcolor{gray!10} & \cellcolor{gray!10}\llap{87.0}
 & \cellcolor{gray!10} & \cellcolor{gray!10}\llap{90.0}
 & \cellcolor{gray!10} & \cellcolor{gray!10}\llap{93.0}
 & \cellcolor{gray!10} & \cellcolor{gray!10}\llap{94.0}
 & \cellcolor{gray!10} & \cellcolor{gray!10}\llap{\textbf{94.8}} \\
 & & \textcolor{black}{ReForge-DA} & - 
 & 92.3 & \up{41.7} & 92.8 & \up{9.4} & 93.1 & \up{7.1} & 92.7 & \up{2.9} & 94.4 & \up{1.6} & 94.7 & \up{0.7} & \textbf{95.1} & \up{0.3} \\
 & & \textcolor{black}{ReForge-DF}     & - 
 & 94.4 & \up{44.9} & 94.2 & \up{11.0} & 94.2 & \up{8.3} & 94.0 & \up{4.5} & 94.4 & \up{1.5} & 94.3 & \up{0.3} & \textbf{95.0} & \up{0.2} \\
\cmidrule{2-18}
 & \multirow{3}{*}{14}
 & \cellcolor{gray!10}anchor   & \cellcolor{gray!10}94.3 
 & \cellcolor{gray!10} & \cellcolor{gray!10}\llap{68.5}
 & \cellcolor{gray!10} & \cellcolor{gray!10}\llap{79.4}
 & \cellcolor{gray!10} & \cellcolor{gray!10}\llap{80.2}
 & \cellcolor{gray!10} & \cellcolor{gray!10}\llap{85.5}
 & \cellcolor{gray!10} & \cellcolor{gray!10}\llap{89.1}
 & \cellcolor{gray!10} & \cellcolor{gray!10}\llap{90.5}
 & \cellcolor{gray!10} & \cellcolor{gray!10}\llap{\textbf{91.0}} \\
 & & \textcolor{black}{ReForge-DA} & - 
 & 88.0 & \up{28.5} & 88.3 & \up{11.2} & 88.3 & \up{10.1} & 88.2 & \up{3.2} & 91.6 & \up{2.7} & 91.7 & \up{1.3} & \textbf{92.2} & \up{1.4} \\
 & & \textcolor{black}{ReForge-DF}     & - 
 & 90.9 & \up{32.8} & 90.9 & \up{14.5} & 90.8 & \up{13.3} & 90.6 & \up{6.0} & 91.3 & \up{2.4} & 91.0 & \up{0.5} & \textbf{92.1} & \up{1.2} \\
\cmidrule{2-18}
 & \multirow{3}{*}{20}
 & \cellcolor{gray!10}anchor   & \cellcolor{gray!10}94.7 
 & \cellcolor{gray!10} & \cellcolor{gray!10}\llap{65.4}
 & \cellcolor{gray!10} & \cellcolor{gray!10}\llap{74.0}
 & \cellcolor{gray!10} & \cellcolor{gray!10}\llap{76.7}
 & \cellcolor{gray!10} & \cellcolor{gray!10}\llap{82.8}
 & \cellcolor{gray!10} & \cellcolor{gray!10}\llap{87.7}
 & \cellcolor{gray!10} & \cellcolor{gray!10}\llap{88.4}
 & \cellcolor{gray!10} & \cellcolor{gray!10}\llap{\textbf{90.1}} \\
 & & \textcolor{black}{ReForge-DA} & - 
 & 86.0 & \up{31.5} & 86.2 & \up{16.5} & 86.4 & \up{12.6} & 86.2 & \up{4.1} & 90.9 & \up{3.6} & 90.8 & \up{2.7} & \textbf{91.5} & \up{1.5} \\
 & & \textcolor{black}{ReForge-DF}     & - 
 & 89.8 & \up{37.3} & 89.6 & \up{21.1} & 89.5 & \up{16.6} & 89.4 & \up{7.9} & 90.3 & \up{3.0} & 89.7 & \up{1.5} & \textbf{91.1} & \up{1.1} \\

\midrule
\rowcolor{LightCyan}
\multicolumn{2}{l}{} & \textcolor{black}{ReForge-DA} & - 
& & \llap{\upavg{42.3}} & & \llap{\upavg{17.4}} & & \llap{\upavg{13.1}} & & \llap{\upavg{3.7}} & & \llap{\upavg{3.9}} & & \llap{\upavg{3.4}} & & \llap{\upavg{3.2}} \\
\rowcolor{LightCyan}
\multicolumn{2}{l}{\multirow{-2}{*}{\textbf{Avg. Improvement}}} & \textcolor{black}{ReForge-DF} & - 
& & \llap{\upavg{48.0}} & & \llap{\upavg{21.1}} & & \llap{\upavg{16.2}} & & \llap{\upavg{7.3}} & & \llap{\upavg{2.7}} & & \llap{\upavg{1.2}} & & \llap{\upavg{2.3}} \\

\bottomrule
\end{tabular}
}
\end{table*}

%% file: tables/main_language.tex
\begin{table*}[t]
\caption{Model merging on language tasks with Llama architectures. The blue subscript values indicate the \textbf{relative percentage improvement} compared to the corresponding \textit{anchor}.}
\label{tab:nlp_main}
\centering

\definecolor{DeepBlue}{HTML}{00509E}
\definecolor{LightCyan}{rgb}{0.88,1,1}

\newcommand{\nlpup}[1]{\textsubscript{\textcolor{blue}{+#1}}}
\newcommand{\nlpdown}[1]{\textsubscript{\textcolor{red}{-#1}}}
\newcommand{\nlpupavg}[1]{{\footnotesize+#1}}

\setlength{\tabcolsep}{3.5pt}

\resizebox{\textwidth}{!}{
\begin{tabular}{ll c *{7}{r@{}l}}
\toprule
\textbf{Model} & \textbf{Setting} & \textbf{Indiv.} & 
\multicolumn{2}{c}{\textbf{Pretrained}} & \multicolumn{2}{c}{\textbf{TA}} & 
\multicolumn{2}{c}{\textbf{TIES}} & \multicolumn{2}{c}{\textbf{RegMean}} & 
\multicolumn{2}{c}{\textbf{TSV}} & \multicolumn{2}{c}{\textbf{WUDI}} & 
\multicolumn{2}{c}{\textbf{ISO-CTS}} \\
\midrule

\multirow{3}{*}{Llama-3.2-3B}
 & \cellcolor{gray!10}anchor   & \cellcolor{gray!10}0.499 
 & \cellcolor{gray!10} & \cellcolor{gray!10}\llap{0.301} 
 & \cellcolor{gray!10} & \cellcolor{gray!10}\llap{0.438} 
 & \cellcolor{gray!10} & \cellcolor{gray!10}\llap{0.435} 
 & \cellcolor{gray!10} & \cellcolor{gray!10}\llap{0.377} 
 & \cellcolor{gray!10} & \cellcolor{gray!10}\llap{\textbf{0.471}} 
 & \cellcolor{gray!10} & \cellcolor{gray!10}\llap{0.461} 
 & \cellcolor{gray!10} & \cellcolor{gray!10}\llap{0.441} \\
\cmidrule{2-17}
 & \textcolor{black}{ReForge-DA} & - 
 & \textcolor{black}{0.342} & \nlpup{\textcolor{black}{13.6}} & 0.448 & \nlpup{2.3} & 0.447 & \nlpup{2.8} & 0.406 & \nlpup{7.7} & 0.478 & \nlpup{1.5} & \textbf{0.479} & \nlpup{3.9} & 0.454 & \nlpup{2.9} \\
 & \textcolor{black}{ReForge-DF}     & - 
 & \textcolor{black}{0.450} & \nlpup{\textcolor{black}{49.3}} & 0.455 & \nlpup{3.9} & 0.468 & \nlpup{7.6} & 0.456 & \nlpup{21.0} & \textbf{0.486} & \nlpup{3.2} & 0.478 & \nlpup{3.7} & 0.457 & \nlpup{3.6} \\

\midrule

\multirow{3}{*}{Llama-3.1-8B}
 & \cellcolor{gray!10}anchor   & \cellcolor{gray!10}0.630 
 & \cellcolor{gray!10} & \cellcolor{gray!10}\llap{0.385} 
 & \cellcolor{gray!10} & \cellcolor{gray!10}\llap{0.541} 
 & \cellcolor{gray!10} & \cellcolor{gray!10}\llap{0.554} 
 & \cellcolor{gray!10} & \cellcolor{gray!10}\llap{0.526} 
 & \cellcolor{gray!10} & \cellcolor{gray!10}\llap{0.557} 
 & \cellcolor{gray!10} & \cellcolor{gray!10}\llap{\textbf{0.560}} 
 & \cellcolor{gray!10} & \cellcolor{gray!10}\llap{0.556} \\
\cmidrule{2-17}
 & \textcolor{black}{ReForge-DA} & - 
 & \textcolor{black}{0.484} & \nlpup{\textcolor{black}{25.7}} & 0.568 & \nlpup{5.0} & \textbf{0.579} & \nlpup{4.5} & 0.530 & \nlpup{0.8} & 0.574 & \nlpup{3.1} & 0.564 & \nlpup{0.7} & 0.576 & \nlpup{3.6} \\
 & \textcolor{black}{ReForge-DF}     & - 
 & \textcolor{black}{0.507} & \nlpup{\textcolor{black}{31.7}} & 0.568 & \nlpup{5.0} & 0.558 & \nlpup{0.7} & 0.528 & \nlpup{0.4} & \textbf{0.573} & \nlpup{2.9} & 0.561 & \nlpup{0.2} & 0.564 & \nlpup{1.4} \\

\midrule


\rowcolor{LightCyan}
 & \textcolor{black}{ReForge-DA} & - 
 & & \llap{\nlpupavg{\textcolor{black}{19.6}}} & & \llap{\nlpupavg{3.6}} & & \llap{\nlpupavg{3.6}} & & \llap{\nlpupavg{4.2}} & & \llap{\nlpupavg{2.3}} & & \llap{\nlpupavg{2.3}} & & \llap{\nlpupavg{3.3}} \\

\rowcolor{LightCyan}
\multirow{-2}{*}{\textbf{Avg. Improvement}} & \textcolor{black}{ReForge-DF}     & - 
 & & \llap{\nlpupavg{\textcolor{black}{40.5}}} & & \llap{\nlpupavg{4.4}} & & \llap{\nlpupavg{4.2}} & & \llap{\nlpupavg{10.7}} & & \llap{\nlpupavg{3.0}} & & \llap{\nlpupavg{1.9}} & & \llap{\nlpupavg{2.5}} \\


\bottomrule
\end{tabular}
}\vspace{0pt}
\end{table*}

%% file: tables/ablation_BO.tex
\begin{table}[t]
\centering
\caption{Different hyperparameter search methods (\textit{shared-}$\lambda$, \textit{random search}, and \textit{BO}) and their impacts on 8-, 14-, and 20-task merging with ViT-B/32 architecture. Results for \textcolor{black}{ReForge-DA} and \textcolor{black}{ReForge-DF} are reported as mean $\pm$ std over \textcolor{black}{five} seeds.}
\label{tab:bo_ablation}
\vspace{2pt}
\resizebox{\textwidth}{!}{
\small
\setlength{\tabcolsep}{3pt} 
\begin{tabular}{ll |ccc |ccc |ccc}
\toprule
\multirow{2}{*}{\textbf{Setting}} & \multirow{2}{*}{\textbf{Variant}} & \multicolumn{3}{c}{\textbf{8 Tasks}} & \multicolumn{3}{|c}{\textbf{14 Tasks}} & \multicolumn{3}{|c}{\textbf{20 Tasks}} \\
\cmidrule(lr){3-5} \cmidrule(lr){6-8} \cmidrule(lr){9-11}
& & \textbf{TSV} & \textbf{WUDI} & \textbf{ISO-CTS} & \textbf{TSV} & \textbf{WUDI} & \textbf{ISO-CTS} & \textbf{TSV} & \textbf{WUDI} & \textbf{ISO-CTS} \\
\midrule
anchor & - & 85.9 & 87.0 & 86.4 & 79.9 & 80.5 & 81.5 & 76.9 & 76.1 & 77.6 \\
\midrule
\multirow{3}{*}{\textcolor{black}{ReForge-DA}}
& shared-$\lambda$ & $88.1_{\pm0.01}$ & $89.0_{\pm0.01}$ & $89.3_{\pm0.01}$ & $83.2_{\pm0.00}$ & $83.5_{\pm0.00}$ & $83.9_{\pm0.00}$ & $80.1_{\pm0.01}$ & $80.6_{\pm0.01}$ & $80.3_{\pm0.02}$ \\
& random            & $88.6_{\pm0.09}$ & $89.2_{\pm0.02}$ & $89.7_{\pm0.24}$ & $83.7_{\pm0.11}$ & $83.9_{\pm0.30}$ & $84.9_{\pm0.07}$ & $80.6_{\pm0.18}$ & $81.2_{\pm0.32}$ & $81.7_{\pm0.37}$ \\
& \cellcolor{LightCyan} BO & \cellcolor{LightCyan} $\mathbf{88.9}_{\pm\mathbf{0.05}}$ & \cellcolor{LightCyan} $\mathbf{89.4}_{\pm\mathbf{0.02}}$ & \cellcolor{LightCyan} $\mathbf{90.2}_{\pm\mathbf{0.03}}$ & \cellcolor{LightCyan} $\mathbf{84.3}_{\pm\mathbf{0.04}}$ & \cellcolor{LightCyan} $\mathbf{84.6}_{\pm\mathbf{0.04}}$ & \cellcolor{LightCyan} $\mathbf{85.5}_{\pm\mathbf{0.06}}$ & \cellcolor{LightCyan} $\mathbf{82.0}_{\pm\mathbf{0.04}}$ & \cellcolor{LightCyan} $\mathbf{82.2}_{\pm\mathbf{0.06}}$ & \cellcolor{LightCyan} $\mathbf{82.8}_{\pm\mathbf{0.05}}$ \\
\midrule
\multirow{3}{*}{\textcolor{black}{ReForge-DF}}
& shared-$\lambda$ & $87.1_{\pm0.00}$ & $87.1_{\pm0.02}$ & $88.0_{\pm0.00}$ & $81.5_{\pm0.00}$ & $80.7_{\pm0.00}$ & $82.7_{\pm0.00}$ & $77.9_{\pm0.00}$ & $76.7_{\pm0.01}$ & $78.8_{\pm0.00}$ \\
& random            & $87.4_{\pm0.20}$ & $87.4_{\pm0.13}$ & $88.1_{\pm0.23}$ & $82.2_{\pm0.24}$ & $81.2_{\pm0.08}$ & $83.3_{\pm0.18}$ & $79.0_{\pm0.18}$ & $77.3_{\pm0.20}$ & $80.0_{\pm0.33}$ \\
& \cellcolor{LightCyan} BO & \cellcolor{LightCyan} $\mathbf{87.8}_{\pm\mathbf{0.04}}$ & \cellcolor{LightCyan} $\mathbf{87.6}_{\pm\mathbf{0.04}}$ & \cellcolor{LightCyan} $\mathbf{89.0}_{\pm\mathbf{0.03}}$ & \cellcolor{LightCyan} $\mathbf{82.8}_{\pm\mathbf{0.03}}$ & \cellcolor{LightCyan} $\mathbf{81.7}_{\pm\mathbf{0.06}}$ & \cellcolor{LightCyan} $\mathbf{84.3}_{\pm\mathbf{0.05}}$ & \cellcolor{LightCyan} $\mathbf{79.7}_{\pm\mathbf{0.05}}$ & \cellcolor{LightCyan} $\mathbf{78.0}_{\pm\mathbf{0.02}}$ & \cellcolor{LightCyan} $\mathbf{81.5}_{\pm\mathbf{0.04}}$ \\
\bottomrule
\end{tabular}
}
\end{table}

%% file: tables/fewshot.tex
\begin{table}[t]
\centering
\caption{\textcolor{black}{Few-shot Gram mixing on ViT-B/32 (test accuracy, \%). Activation-only and mixed variants use one or five unlabeled calibration examples per task, with the mix weight selected by BO on separate validation data. ReForge-DA with 128 examples per task provides a reference. Subscripts indicate sample standard deviations over five seeds with fixed checkpoints and data.}}
\label{tab:gram_ablation}
\vspace{-4pt}
\resizebox{\textwidth}{!}{
\small
\setlength{\tabcolsep}{3pt}
\begin{tabular}{l ccc ccc ccc}
\toprule
\multirow{2}{*}{\textbf{Setting}} & \multicolumn{3}{c}{\textbf{8 Tasks}} & \multicolumn{3}{c}{\textbf{14 Tasks}} & \multicolumn{3}{c}{\textbf{20 Tasks}} \\
\cmidrule(lr){2-4} \cmidrule(lr){5-7} \cmidrule(lr){8-10}
& \textbf{TSV} & \textbf{WUDI} & \textbf{ISO-CTS} & \textbf{TSV} & \textbf{WUDI} & \textbf{ISO-CTS} & \textbf{TSV} & \textbf{WUDI} & \textbf{ISO-CTS} \\
\midrule
anchor      & 85.9 & 87.0 & 86.4 & 79.9 & 80.5 & 81.5 & 76.9 & 76.1 & 77.6 \\
\textcolor{black}{ReForge-DF}   & $87.8_{\pm0.04}$ & $87.6_{\pm0.04}$ & $89.0_{\pm0.03}$ & $82.8_{\pm0.03}$ & $81.7_{\pm0.06}$ & $84.3_{\pm0.05}$ & $79.7_{\pm0.05}$ & $78.0_{\pm0.02}$ & $81.5_{\pm0.04}$ \\
\midrule
1-shot      & $87.8_{\pm0.06}$ & $\mathbf{89.0}_{\pm\mathbf{0.02}}$ & $88.7_{\pm0.03}$ & $83.2_{\pm0.15}$ & $\mathbf{84.1}_{\pm\mathbf{0.06}}$ & $84.2_{\pm0.07}$ & $80.6_{\pm0.06}$ & $\mathbf{81.1}_{\pm\mathbf{0.04}}$ & $80.5_{\pm0.06}$ \\
1-shot mix  & $\mathbf{88.9}_{\pm\mathbf{0.05}}$ & $\mathbf{89.0}_{\pm\mathbf{0.08}}$ & $\mathbf{89.7}_{\pm\mathbf{0.05}}$ & $\mathbf{84.3}_{\pm\mathbf{0.07}}$ & $\mathbf{84.1}_{\pm\mathbf{0.05}}$ & $\mathbf{85.7}_{\pm\mathbf{0.04}}$ & $\mathbf{81.5}_{\pm\mathbf{0.15}}$ & $\mathbf{81.1}_{\pm\mathbf{0.03}}$ & $\mathbf{82.8}_{\pm\mathbf{0.03}}$ \\
\midrule
5-shot      & $88.5_{\pm0.11}$ & $\mathbf{89.3}_{\pm\mathbf{0.06}}$ & $89.5_{\pm0.03}$ & $84.1_{\pm0.04}$ & $\mathbf{84.6}_{\pm\mathbf{0.07}}$ & $85.3_{\pm0.04}$ & $81.7_{\pm0.02}$ & $\mathbf{81.9}_{\pm\mathbf{0.04}}$ & $82.2_{\pm0.02}$ \\
5-shot mix  & $\mathbf{89.4}_{\pm\mathbf{0.11}}$ & $\mathbf{89.3}_{\pm\mathbf{0.07}}$ & $\mathbf{90.2}_{\pm\mathbf{0.02}}$ & $\mathbf{84.9}_{\pm\mathbf{0.11}}$ & $\mathbf{84.6}_{\pm\mathbf{0.08}}$ & $\mathbf{86.1}_{\pm\mathbf{0.06}}$ & $\mathbf{82.4}_{\pm\mathbf{0.03}}$ & $\mathbf{81.9}_{\pm\mathbf{0.03}}$ & $\mathbf{83.7}_{\pm\mathbf{0.02}}$ \\
\midrule
\textcolor{black}{128-sample ReForge-DA} & \textcolor{black}{$88.9_{\pm0.05}$} & \textcolor{black}{$89.4_{\pm0.02}$} & \textcolor{black}{$90.2_{\pm0.03}$} & \textcolor{black}{$84.3_{\pm0.04}$} & \textcolor{black}{$84.6_{\pm0.04}$} & \textcolor{black}{$85.5_{\pm0.06}$} & \textcolor{black}{$82.0_{\pm0.04}$} & \textcolor{black}{$82.2_{\pm0.06}$} & \textcolor{black}{$82.8_{\pm0.05}$} \\
\bottomrule
\end{tabular}
}
\vspace{-8pt}
\end{table}

%% file: sections/ai_use_statement.tex
\section*{\textcolor{black}{AI Use Statement}}
\textcolor{black}{Generative AI tools assisted with manuscript drafting and editing, as well as LaTeX formatting checks. The authors are responsible for the final content of this paper.}

%% file: sections/gram_analysis.tex
\section{\textcolor{black}{Empirical Diagnostics of the Task-Vector Gram Surrogate}}
\label{sec:appendix_a}
\label{app:proof}
\input{sections/surrogate_diagnostics}

\input{sections/gram_controls}

%% file: sections/surrogate_diagnostics.tex
\subsection{\textcolor{black}{Local Subspace and Output Diagnostics}}
\label{app:surrogate_diagnostics}
\begingroup
\textcolor{black}{We compare task-vector Grams with Gaussian random-PSD and isotropic controls on ViT-B/32, using 8, 14, and 20 tasks with ISO-CTS and Task Arithmetic anchors. We select the ridge strength and correction scale using calibration activations, then evaluate local-output agreement on held-out activations.}

\paragraph{\textcolor{black}{Data and evaluation protocol.}}
\textcolor{black}{For each task, we exclude the fixed BO-validation examples and shuffle the remaining training images to form a calibration set of 128 images and a held-out set of another 128 images. We repeat this procedure with five seeds. Calibration and held-out sets are disjoint within each split, although images may recur across splits. Model checkpoints remain fixed, and each task's splits are shared across anchors and task-count settings.}

\textcolor{black}{We collect input activations from all 48 linear modules, including the attention output projections. Each held-out image contributes 50 token activations, giving 6,400 activation vectors per task and module. Calibration activations are used for parameter selection; held-out activations are reserved for evaluating local-output agreement.}

\paragraph{\textcolor{black}{Gram construction and local regression.}}
\textcolor{black}{Suppress the module index and write $\mathbf V_t=\mathbf U^{(t)}$ and $\mathbf U_0=\mathbf U^{(0)}$. With activations stored column-wise as in the main text, define}
\begingroup\color{black}\begin{equation}
\mathbf H_t^{\mathrm{act}}=\mathbf X_t^{\mathrm{cal}}(\mathbf X_t^{\mathrm{cal}})^\top,
\quad \mathbf H_t^{\mathrm{tv}}=\mathbf V_t^\top\mathbf V_t,
\quad \widetilde{\mathbf H}_t^q=\frac{\mathbf H_t^q}{\operatorname{tr}(\mathbf H_t^q)},
\quad q\in\{\mathrm{act},\mathrm{tv},\mathrm{rnd},\mathrm{iso}\}.
\label{eq:diag_grams}
\end{equation}\endgroup
\textcolor{black}{The random control uses $\mathbf H_t^{\mathrm{rnd}}=\mathbf R_t^\top\mathbf R_t$, with independent standard Gaussian entries in $\mathbf R_t\in\mathbb R^{d_{\mathrm{out}}\times d_{\mathrm{in}}}$. Ten fixed random draws are generated per task and module; their realizations are shared across anchors. The isotropic control uses $\widetilde{\mathbf H}_t^{\mathrm{iso}}=\mathbf I/d_{\mathrm{in}}$ for each task. Trace normalization is applied separately to each task Gram, before summing:}
\begingroup\color{black}\begin{equation}
\mathbf G_q=\sum_t\widetilde{\mathbf H}_t^q,\qquad
\mathbf C_q=\sum_t\mathbf V_t\widetilde{\mathbf H}_t^q,\qquad
\mathbf Z_q(\lambda)=(\mathbf C_q-\mathbf U_0\mathbf G_q)(\mathbf G_q+\lambda\mathbf I)^{-1}.
\label{eq:diag_centered}
\end{equation}\endgroup
\textcolor{black}{Here $\mathbf Z_q$ is the} \emph{\textcolor{black}{anchor-centered correction}} \textcolor{black}{of the corresponding ridge solution.}

\paragraph{\textcolor{black}{Local parameter selection.}}
\textcolor{black}{For each module, anchor, and split, select a shared ridge strength and a nonnegative correction scale from calibration data:}
\begingroup\color{black}\begin{equation}
(\lambda_*,a_*)\in\arg\min_{\lambda\in\Lambda,\ a\ge0}
\frac{\|\mathbf Z_{\mathrm{act}}(\lambda)-a\mathbf Z_{\mathrm{tv}}(\lambda)\|_F}
{\|\mathbf Z_{\mathrm{act}}(\lambda)\|_F},
\qquad \Lambda=\{10^{-6+8j/256}:j=0,\ldots,256\}.
\label{eq:diag_local_fit}
\end{equation}\endgroup
\textcolor{black}{For each $\lambda$, the optimal scale is
$a(\lambda)=\max\{0,\langle\mathbf Z_{\mathrm{tv}},\mathbf Z_{\mathrm{act}}\rangle_F/\|\mathbf Z_{\mathrm{tv}}\|_F^2\}$. Both the random-PSD and isotropic output controls use the task-vector-fitted $\lambda_*$ and $a_*$.}

\paragraph{\textcolor{black}{Dominant eigenspace overlap.}}
\textcolor{black}{Let $\mathbf Q_q$ contain the top $k=32$ orthonormal eigenvectors of $\mathbf G_q$. We report}
\begingroup\color{black}\begin{equation}
\mathrm{Overlap}_{32}(q)=\frac{1}{32}\|\mathbf Q_{\mathrm{act}}^\top\mathbf Q_q\|_F^2.
\label{eq:diag_overlap}
\end{equation}\endgroup
\textcolor{black}{This statistic does not depend on the anchor or fitted parameters. The isotropic Gram has no unique top-32 eigenspace because all eigenvalues are equal; its overlap is therefore reported as a random-basis Monte Carlo reference. For each module, we take the median overlap of ten Haar-random 32-dimensional subspaces with a fixed reference subspace, then average across modules. We report the mean and sample standard deviation over five Monte Carlo repetitions.}

\paragraph{\textcolor{black}{Ridge-filtered operator error.}}
\textcolor{black}{Define the ridge-filtered operator $\mathbf F_\lambda(\mathbf G)=\mathbf G(\mathbf G+\lambda\mathbf I)^{-1}$. For each candidate, fit its nonnegative Gram scale
$c_q=\max\{0,\langle\mathbf G_{\mathrm{act}},\mathbf G_q\rangle_F/\|\mathbf G_q\|_F^2\}$ and compute}
\begingroup\color{black}\begin{equation}
E_{\mathrm{op}}(q)=
\frac{\|\mathbf F_{\lambda_*}(\mathbf G_{\mathrm{act}})-\mathbf F_{\lambda_*}(c_q\mathbf G_q)\|_F}
{\|\mathbf F_{\lambda_*}(\mathbf G_{\mathrm{act}})\|_F}.
\label{eq:diag_operator}
\end{equation}\endgroup
\textcolor{black}{The ridge strength is fixed by Eq.~\eqref{eq:diag_local_fit}; each surrogate, including each random draw, receives its own prescribed $c_q$. The Gram scale $c_q$ differs from the correction scale $a_*$. The calculation uses float64 eigendecompositions and spectral filtering.}

\paragraph{\textcolor{black}{Held-out local-output error.}}
\textcolor{black}{Regression Grams are trace-normalized per task; held-out Grams retain their original scale for all three surrogate evaluations. Let $\mathbf H_{\mathrm{hold}}=\sum_t\mathbf X_t^{\mathrm{hold}}(\mathbf X_t^{\mathrm{hold}})^\top$ be the sum of} \emph{\textcolor{black}{unnormalized}} \textcolor{black}{held-out Grams and $\mathbf D_q=\mathbf Z_{\mathrm{act}}(\lambda_*)-a_*\mathbf Z_q(\lambda_*)$. The output diagnostic is}
\begingroup\color{black}\begin{equation}
E_{\mathrm{out}}(q)=\left[
\frac{\operatorname{tr}(\mathbf D_q\mathbf H_{\mathrm{hold}}\mathbf D_q^\top)}
{\operatorname{tr}(\mathbf Z_{\mathrm{act}}(\lambda_*)\mathbf H_{\mathrm{hold}}\mathbf Z_{\mathrm{act}}(\lambda_*)^\top)}
\right]^{1/2}.
\label{eq:diag_output}
\end{equation}\endgroup
\textcolor{black}{This measures disagreement between local} \emph{\textcolor{black}{anchor-centered residual outputs}}\textcolor{black}{.}

\paragraph{\textcolor{black}{Summary statistics.}}
\textcolor{black}{For each data split, we average each metric over the 48 modules. For the random-PSD control, we first take the median over ten random draws within each module. We then report the mean and sample standard deviation across the five data splits. The isotropic subspace-overlap reference is summarized separately, with its standard deviation reflecting random-basis sampling.}

\begin{table}[h]
\centering
\caption{\textcolor{black}{Task-vector Gram diagnostics with random-PSD and isotropic controls. Values are means with subscript sample standard deviations over five seeds, averaging 48 modules per seed. Random PSD uses the median of ten draws within each module. $\dagger$ marks the isotropic random-basis reference, whose SD reflects Monte Carlo sampling rather than data splits.}}
\label{tab:surrogate_diagnostics}
\setlength{\tabcolsep}{3pt}
\small
\begin{tabular}{cllccc}
\toprule
\multicolumn{1}{c}{\textcolor{black}{Tasks}} & \multicolumn{1}{c}{\textcolor{black}{Anchor}} & \multicolumn{1}{c}{\textcolor{black}{Gram}} & \makecell{\textcolor{black}{Top-32}\\\textcolor{black}{overlap $\uparrow$}} & \makecell{\textcolor{black}{Ridge-filtered}\\\textcolor{black}{operator}\\\textcolor{black}{error $\downarrow$}} & \makecell{\textcolor{black}{Held-out local}\\\textcolor{black}{output}\\\textcolor{black}{error $\downarrow$}}\\
\midrule
\textcolor{black}{8} & \textcolor{black}{ISO-CTS} & \textcolor{black}{Task vector} & \textcolor{black}{$0.3733_{\scriptstyle\pm 0.0011}$} & \textcolor{black}{$0.4367_{\scriptstyle\pm 0.0048}$} & \textcolor{black}{$0.6069_{\scriptstyle\pm 0.0009}$}\\
 &  & \textcolor{black}{Random PSD} & \textcolor{black}{$0.0338_{\scriptstyle\pm <0.0001}$} & \textcolor{black}{$0.5986_{\scriptstyle\pm 0.0042}$} & \textcolor{black}{$0.7344_{\scriptstyle\pm 0.0025}$}\\
 &  & \textcolor{black}{Isotropic} & \textcolor{black}{$0.0339_{\scriptstyle\pm <0.0001}^{\dagger}$} & \textcolor{black}{$0.6111_{\scriptstyle\pm 0.0038}$} & \textcolor{black}{$0.7318_{\scriptstyle\pm 0.0025}$}\\
\midrule
\textcolor{black}{8} & \textcolor{black}{TA} & \textcolor{black}{Task vector} & \textcolor{black}{$0.3733_{\scriptstyle\pm 0.0011}$} & \textcolor{black}{$0.7307_{\scriptstyle\pm 0.0023}$} & \textcolor{black}{$0.8164_{\scriptstyle\pm 0.0009}$}\\
 &  & \textcolor{black}{Random PSD} & \textcolor{black}{$0.0338_{\scriptstyle\pm <0.0001}$} & \textcolor{black}{$0.9111_{\scriptstyle\pm 0.0016}$} & \textcolor{black}{$1.0040_{\scriptstyle\pm 0.0001}$}\\
 &  & \textcolor{black}{Isotropic} & \textcolor{black}{$0.0339_{\scriptstyle\pm <0.0001}^{\dagger}$} & \textcolor{black}{$0.9185_{\scriptstyle\pm 0.0014}$} & \textcolor{black}{$1.0040_{\scriptstyle\pm 0.0001}$}\\
\midrule
\textcolor{black}{14} & \textcolor{black}{ISO-CTS} & \textcolor{black}{Task vector} & \textcolor{black}{$0.3884_{\scriptstyle\pm 0.0013}$} & \textcolor{black}{$0.4224_{\scriptstyle\pm 0.0006}$} & \textcolor{black}{$0.5977_{\scriptstyle\pm 0.0004}$}\\
 &  & \textcolor{black}{Random PSD} & \textcolor{black}{$0.0340_{\scriptstyle\pm <0.0001}$} & \textcolor{black}{$0.6003_{\scriptstyle\pm 0.0008}$} & \textcolor{black}{$0.7169_{\scriptstyle\pm 0.0009}$}\\
 &  & \textcolor{black}{Isotropic} & \textcolor{black}{$0.0339_{\scriptstyle\pm <0.0001}^{\dagger}$} & \textcolor{black}{$0.6082_{\scriptstyle\pm 0.0008}$} & \textcolor{black}{$0.7153_{\scriptstyle\pm 0.0009}$}\\
\midrule
\textcolor{black}{14} & \textcolor{black}{TA} & \textcolor{black}{Task vector} & \textcolor{black}{$0.3884_{\scriptstyle\pm 0.0013}$} & \textcolor{black}{$0.7187_{\scriptstyle\pm 0.0017}$} & \textcolor{black}{$0.8025_{\scriptstyle\pm 0.0008}$}\\
 &  & \textcolor{black}{Random PSD} & \textcolor{black}{$0.0339_{\scriptstyle\pm <0.0001}$} & \textcolor{black}{$0.8805_{\scriptstyle\pm 0.0024}$} & \textcolor{black}{$1.0032_{\scriptstyle\pm <0.0001}$}\\
 &  & \textcolor{black}{Isotropic} & \textcolor{black}{$0.0339_{\scriptstyle\pm <0.0001}^{\dagger}$} & \textcolor{black}{$0.8850_{\scriptstyle\pm 0.0019}$} & \textcolor{black}{$1.0032_{\scriptstyle\pm <0.0001}$}\\
\midrule
\textcolor{black}{20} & \textcolor{black}{ISO-CTS} & \textcolor{black}{Task vector} & \textcolor{black}{$0.3890_{\scriptstyle\pm 0.0008}$} & \textcolor{black}{$0.4242_{\scriptstyle\pm 0.0009}$} & \textcolor{black}{$0.6056_{\scriptstyle\pm 0.0006}$}\\
 &  & \textcolor{black}{Random PSD} & \textcolor{black}{$0.0338_{\scriptstyle\pm <0.0001}$} & \textcolor{black}{$0.6036_{\scriptstyle\pm 0.0012}$} & \textcolor{black}{$0.7146_{\scriptstyle\pm 0.0006}$}\\
 &  & \textcolor{black}{Isotropic} & \textcolor{black}{$0.0339_{\scriptstyle\pm <0.0001}^{\dagger}$} & \textcolor{black}{$0.6084_{\scriptstyle\pm 0.0012}$} & \textcolor{black}{$0.7137_{\scriptstyle\pm 0.0006}$}\\
\midrule
\textcolor{black}{20} & \textcolor{black}{TA} & \textcolor{black}{Task vector} & \textcolor{black}{$0.3890_{\scriptstyle\pm 0.0008}$} & \textcolor{black}{$0.6928_{\scriptstyle\pm 0.0012}$} & \textcolor{black}{$0.7703_{\scriptstyle\pm 0.0006}$}\\
 &  & \textcolor{black}{Random PSD} & \textcolor{black}{$0.0338_{\scriptstyle\pm <0.0001}$} & \textcolor{black}{$0.8954_{\scriptstyle\pm 0.0013}$} & \textcolor{black}{$0.9989_{\scriptstyle\pm 0.0001}$}\\
 &  & \textcolor{black}{Isotropic} & \textcolor{black}{$0.0339_{\scriptstyle\pm <0.0001}^{\dagger}$} & \textcolor{black}{$0.8979_{\scriptstyle\pm 0.0013}$} & \textcolor{black}{$0.9989_{\scriptstyle\pm 0.0001}$}\\
\bottomrule
\end{tabular}
\end{table}

\textcolor{black}{The task-vector surrogate has higher subspace overlap than the random-PSD control and lower operator and output errors than both controls in all six task}--\textcolor{black}{anchor settings} (\textcolor{black}{Table~\ref{tab:surrogate_diagnostics}})\textcolor{black}{. Its overlap also exceeds the isotropic random-basis reference. The random-PSD and isotropic controls yield similar output errors when evaluated with the task-vector-fitted parameters.}
\par\endgroup

%% file: sections/gram_controls.tex
\begingroup
\subsection{\textcolor{black}{Task-Vector Grams versus Uninformative Controls}}
\label{app:gram_controls}

\paragraph{\textcolor{black}{Setup.}}
\textcolor{black}{We assess whether task-specific information in the task-vector Grams contributes to the final performance of ReForge-DF. On ViT-B}/\textcolor{black}{32 with TSV and ISO-CTS anchors, we compare the task-vector Gram with two substitutes: a randomly generated positive semidefinite matrix} (\textcolor{black}{Random PSD})\textcolor{black}{, and the isotropic matrix $\mathbf I/d_{\mathrm{in}}$, where $d_{\mathrm{in}}$ is the module input dimension. The isotropic control weights all input directions equally, while the random PSD control introduces structure without using task-specific weight information. The comparisons use the same checkpoints, validation and test data, normalization convention, search space, and evaluation budget. Results summarize five seeds.}

\begin{table}[htbp]
\centering
\small
\setlength{\tabcolsep}{3pt}
\caption{\textcolor{black}{Gram replacement ablation on ViT-B}/\textcolor{black}{32 under a common evaluation and search protocol. Test accuracy} (\textcolor{black}{\%}) \textcolor{black}{is reported as mean with subscript $\pm$ standard deviation across five seeds. TV Gram denotes the task-vector Gram used by ReForge-DF; the isotropic substitute is $\mathbf I/d_{\mathrm{in}}$. The last two columns report differences between accuracy means in percentage points. Means are rounded to one decimal place; standard deviations and differences are reported to two decimal places. Differences are computed before rounding.}}
\label{tab:gram_controls}
\begin{tabular}{rlrrrcc}
\toprule
\multicolumn{1}{c}{\textcolor{black}{Tasks}} & \multicolumn{1}{c}{\textcolor{black}{Anchor}} & \multicolumn{1}{c}{\textcolor{black}{TV Gram}} & \multicolumn{1}{c}{\textcolor{black}{Random PSD}} & \multicolumn{1}{c}{\textcolor{black}{Isotropic}} & \multicolumn{1}{c}{\textcolor{black}{$\mathrm{TV}-\mathrm{PSD}$}} & \multicolumn{1}{c}{\textcolor{black}{$\mathrm{TV}-\mathrm{Iso.}$}} \\
\midrule
\textcolor{black}{8} & \textcolor{black}{TSV} & \textcolor{black}{$\mathbf{87.8}_{\scriptstyle\pm 0.04}$} & \textcolor{black}{$86.4_{\scriptstyle\pm 0.17}$} & \textcolor{black}{$86.5_{\scriptstyle\pm 0.15}$} & \textcolor{black}{$1.40$} & \textcolor{black}{$1.29$} \\
\textcolor{black}{8} & \textcolor{black}{ISO-CTS} & \textcolor{black}{$\mathbf{89.0}_{\scriptstyle\pm 0.03}$} & \textcolor{black}{$87.8_{\scriptstyle\pm 0.06}$} & \textcolor{black}{$87.9_{\scriptstyle\pm 0.04}$} & \textcolor{black}{$1.16$} & \textcolor{black}{$1.11$} \\
\midrule
\textcolor{black}{14} & \textcolor{black}{TSV} & \textcolor{black}{$\mathbf{82.8}_{\scriptstyle\pm 0.03}$} & \textcolor{black}{$81.5_{\scriptstyle\pm 0.04}$} & \textcolor{black}{$81.6_{\scriptstyle\pm 0.03}$} & \textcolor{black}{$1.29$} & \textcolor{black}{$1.23$} \\
\textcolor{black}{14} & \textcolor{black}{ISO-CTS} & \textcolor{black}{$\mathbf{84.3}_{\scriptstyle\pm 0.05}$} & \textcolor{black}{$83.4_{\scriptstyle\pm 0.02}$} & \textcolor{black}{$83.4_{\scriptstyle\pm 0.04}$} & \textcolor{black}{$0.87$} & \textcolor{black}{$0.87$} \\
\midrule
\textcolor{black}{20} & \textcolor{black}{TSV} & \textcolor{black}{$\mathbf{79.7}_{\scriptstyle\pm 0.05}$} & \textcolor{black}{$78.9_{\scriptstyle\pm 0.02}$} & \textcolor{black}{$78.9_{\scriptstyle\pm 0.03}$} & \textcolor{black}{$0.81$} & \textcolor{black}{$0.81$} \\
\textcolor{black}{20} & \textcolor{black}{ISO-CTS} & \textcolor{black}{$\mathbf{81.5}_{\scriptstyle\pm 0.04}$} & \textcolor{black}{$80.5_{\scriptstyle\pm 0.03}$} & \textcolor{black}{$80.5_{\scriptstyle\pm 0.02}$} & \textcolor{black}{$1.03$} & \textcolor{black}{$1.02$} \\
\bottomrule
\end{tabular}
\end{table}

\paragraph{\textcolor{black}{Results.}}
\textcolor{black}{Table~\ref{tab:gram_controls} compares task-vector Grams with random PSD and isotropic substitutes under a common evaluation and search protocol. Across all six settings, task-vector Grams improve mean test accuracy by 0.81}--\textcolor{black}{1.40 percentage points over random PSD matrices and by 0.81}--\textcolor{black}{1.29 points over isotropic matrices. The two controls differ by at most 0.11 points, with no consistent advantage from random anisotropy. These results show that task-specific structure in the Gram matrices contributes to ReForge-DF refinement beyond the use of an arbitrary PSD weighting matrix.}
\par\endgroup

%% file: sections/appendix.tex
\section{\textcolor{black}{Anchor Rescaling versus Regression Refinement}}
\label{app:core-controls}
\input{sections/additional_controls}

\input{sections/detailed_results}

\input{sections/posterior_calibration}

\input{sections/runtime_details}

\input{sections/experimental_protocol}

\input{sections/radar_results}

%% file: sections/additional_controls.tex
\begingroup
\label{app:capacity_control}
\label{app:estimator_controls}
\label{app:anchor-rescaling}

\paragraph{\textcolor{black}{Anchor-rescaling control.}}
\textcolor{black}{We test whether rescaling an existing anchor alone can explain ReForge's gains. For each target module, the control rescales the anchor task vector instead of computing a regression update. On ViT-B}/\textcolor{black}{32, we report TSV and ISO-CTS anchors for 8}-\textcolor{black}{, 14}-\textcolor{black}{, and 20-task merging. The 12 Transformer layers are partitioned into three consecutive blocks} (\textcolor{black}{layers 0}--\textcolor{black}{3, 4}--\textcolor{black}{7, and 8}--\textcolor{black}{11})\textcolor{black}{, each containing four module groups: attention-in, attention-out, MLP-in, and MLP-out. For each target module $m$ in block $b$ and group $g$, the control uses}
\begingroup\color{black}\begin{equation}
\mathbf W_{\mathrm{ctrl},m}=\mathbf W_{\mathrm{pre},m}+s_b a_{b,g}\mathbf U_m^{(0)},
\qquad s_b\in[1,1.3],\quad a_{b,g}\in[0,2].
\label{eq:anchor_rescale_control}
\end{equation}\endgroup
\textcolor{black}{The control has 15 search parameters: three block-wise scales $s_b$ and twelve group-wise scales $a_{b,g}$. The group factors $a_{b,g}$ apply only to the target modules. For other parameters within block $b$, we scale their offsets from the pretrained model by $s_b$, giving $\theta_{\mathrm{ctrl}}=\theta_{\mathrm{pre}}+s_b(\theta_{\mathrm{anchor}}-\theta_{\mathrm{pre}})$. Parameters outside the blocks retain their anchor values, following the main method's assembly rule.}

\textcolor{black}{We evaluate six task--anchor configurations using five seeds and 200 validation evaluations per run. Three initial configurations set all block scales to one and all group scales to zero, one, or two, respectively; their evaluations count toward the budget. The control matches ReForge's nominal search dimension and evaluation budget, but uses a different parameterization. Each run selects the configuration with the highest mean validation accuracy for test evaluation. We report mean test accuracy and standard deviation across the five seeds, using the corresponding main-experiment results for ReForge.}

\begin{table}[htbp]
\centering
\small
\setlength{\tabcolsep}{3pt}
\caption{\textcolor{black}{Anchor rescaling versus regression refinement on ViT-B/32. Test accuracy (\%) is reported as mean $\pm$ sample standard deviation across five search seeds. $\Delta_{\mathrm{DF}}$ and $\Delta_{\mathrm{DA}}$ denote gains over the rescaling control for ReForge-DF and ReForge-DA, respectively, measured in percentage points and computed before rounding.}}
\label{tab:capacity_control}
\label{tab:anchor-rescaling}
\begin{tabular}{rlrrrrrr}
\toprule
\multicolumn{1}{c}{\textcolor{black}{Tasks}} & \multicolumn{1}{c}{\textcolor{black}{Anchor}} & \multicolumn{1}{c}{\textcolor{black}{Original}} & \multicolumn{1}{c}{\textcolor{black}{Rescale}} & \multicolumn{1}{c}{\textcolor{black}{ReForge-DF}} & \multicolumn{1}{c}{\textcolor{black}{ReForge-DA}} & \multicolumn{1}{c}{\textcolor{black}{$\Delta_{\mathrm{DF}}$}} & \multicolumn{1}{c}{\textcolor{black}{$\Delta_{\mathrm{DA}}$}} \\
\midrule
\textcolor{black}{8} & \textcolor{black}{TSV} & \textcolor{black}{85.9} & \textcolor{black}{$86.6_{\scriptstyle\pm 0.19}$} & \textcolor{black}{$87.8_{\scriptstyle\pm 0.04}$} & \textcolor{black}{$88.9_{\scriptstyle\pm 0.05}$} & \textcolor{black}{$1.2$} & \textcolor{black}{$2.3$} \\
\textcolor{black}{8} & \textcolor{black}{ISO-CTS} & \textcolor{black}{86.4} & \textcolor{black}{$87.6_{\scriptstyle\pm 0.04}$} & \textcolor{black}{$89.0_{\scriptstyle\pm 0.03}$} & \textcolor{black}{$90.2_{\scriptstyle\pm 0.03}$} & \textcolor{black}{$1.4$} & \textcolor{black}{$2.6$} \\
\midrule
\textcolor{black}{14} & \textcolor{black}{TSV} & \textcolor{black}{79.9} & \textcolor{black}{$81.8_{\scriptstyle\pm 0.10}$} & \textcolor{black}{$82.8_{\scriptstyle\pm 0.03}$} & \textcolor{black}{$84.3_{\scriptstyle\pm 0.04}$} & \textcolor{black}{$1.0$} & \textcolor{black}{$2.5$} \\
\textcolor{black}{14} & \textcolor{black}{ISO-CTS} & \textcolor{black}{81.5} & \textcolor{black}{$83.9_{\scriptstyle\pm 0.10}$} & \textcolor{black}{$84.3_{\scriptstyle\pm 0.05}$} & \textcolor{black}{$85.5_{\scriptstyle\pm 0.06}$} & \textcolor{black}{$0.4$} & \textcolor{black}{$1.6$} \\
\midrule
\textcolor{black}{20} & \textcolor{black}{TSV} & \textcolor{black}{76.9} & \textcolor{black}{$78.6_{\scriptstyle\pm 0.21}$} & \textcolor{black}{$79.7_{\scriptstyle\pm 0.05}$} & \textcolor{black}{$82.0_{\scriptstyle\pm 0.04}$} & \textcolor{black}{$1.1$} & \textcolor{black}{$3.4$} \\
\textcolor{black}{20} & \textcolor{black}{ISO-CTS} & \textcolor{black}{77.6} & \textcolor{black}{$80.5_{\scriptstyle\pm 0.06}$} & \textcolor{black}{$81.5_{\scriptstyle\pm 0.04}$} & \textcolor{black}{$82.8_{\scriptstyle\pm 0.05}$} & \textcolor{black}{$1.1$} & \textcolor{black}{$2.4$} \\
\bottomrule
\end{tabular}
\end{table}

\paragraph{\textcolor{black}{Results.}}
\textcolor{black}{ReForge-DA and ReForge-DF outperform the rescaling control in all six configurations, with larger gains from the data-assisted variant (Table~\ref{tab:anchor-rescaling}). Gains range from 1.6 to 3.4 percentage points for ReForge-DA and from 0.4 to 1.4 percentage points for ReForge-DF.}
\par\endgroup

%% file: sections/detailed_results.tex
\section{\textcolor{black}{Vision Results with Standard Deviations}}
\label{app:additional-results}
\label{app:detailed-aggregate}
\input{tables/table1_std}
Table~\ref{tab:std_performance_comparison} provides the comprehensive results corresponding to Table~\ref{tab:merged_performance_comparison}, including \textcolor{black}{sample} standard deviations across five seeds. \textcolor{black}{Most sample standard deviations lie between 0.01 and 0.06 percentage points. The results vary little across repeated searches with fixed checkpoints and data.}

%% file: tables/table1_std.tex
\begin{table}[H]
\caption{Model merging on vision tasks with ViT architectures. \textcolor{black}{Values report mean accuracy over five seeds; subscripts indicate sample standard deviations, omitted from Table~\ref{tab:merged_performance_comparison} for space.} (Indiv.: Individual)}

\label{tab:std_performance_comparison}
\centering

\newcommand{\std}[1]{\textsubscript{\normalcolor{\(\pm#1\)}}}

\setlength{\tabcolsep}{3.5pt}

\resizebox{\textwidth}{!}{
\begin{tabular}{l c l c *{7}{r@{}l}}
\toprule
\textbf{Model} & \textbf{Tasks} & \textbf{Setting} & \textbf{Indiv.} &
\multicolumn{2}{c}{\textbf{Pretrained}} & \multicolumn{2}{c}{\textbf{TA}} &
\multicolumn{2}{c}{\textbf{TIES}} & \multicolumn{2}{c}{\textbf{RegMean}} &
\multicolumn{2}{c}{\textbf{TSV}} & \multicolumn{2}{c}{\textbf{WUDI}} &
\multicolumn{2}{c}{\textbf{ISO-CTS}} \\
\midrule

\multirow{9}{*}{ViT-B/32}
 & \multirow{3}{*}{8}
 & \cellcolor{gray!10}anchor   & \cellcolor{gray!10}92.8 
 & \cellcolor{gray!10} & \cellcolor{gray!10}\llap{47.7}
 & \cellcolor{gray!10} & \cellcolor{gray!10}\llap{70.4}
 & \cellcolor{gray!10} & \cellcolor{gray!10}\llap{75.7}
 & \cellcolor{gray!10} & \cellcolor{gray!10}\llap{82.3}
 & \cellcolor{gray!10} & \cellcolor{gray!10}\llap{85.9}
 & \cellcolor{gray!10} & \cellcolor{gray!10}\llap{\textbf{87.0}}
 & \cellcolor{gray!10} & \cellcolor{gray!10}\llap{86.4} \\
 & & \textcolor{black}{ReForge-DA} & - 
 & 84.8 & \std{0.02} & 85.0 & \std{0.05} & 85.7 & \std{0.03} & 85.2 & \std{0.05} & 88.9 & \std{0.05} & 89.4 & \std{0.02} & \textbf{90.2} & \std{0.03} \\
 & & \textcolor{black}{ReForge-DF}     & - 
 & 87.9 & \std{0.06} & 87.0 & \std{0.05} & 87.0 & \std{0.06} & 87.9 & \std{0.09} & 87.8 & \std{0.04} & 87.6 & \std{0.04} & \textbf{89.0} & \std{0.03} \\
\cmidrule{2-18}
 & \multirow{3}{*}{14}
 & \cellcolor{gray!10}anchor   & \cellcolor{gray!10}90.9 
 & \cellcolor{gray!10} & \cellcolor{gray!10}\llap{56.9}
 & \cellcolor{gray!10} & \cellcolor{gray!10}\llap{65.2}
 & \cellcolor{gray!10} & \cellcolor{gray!10}\llap{68.2}
 & \cellcolor{gray!10} & \cellcolor{gray!10}\llap{76.6}
 & \cellcolor{gray!10} & \cellcolor{gray!10}\llap{79.9}
 & \cellcolor{gray!10} & \cellcolor{gray!10}\llap{80.5}
 & \cellcolor{gray!10} & \cellcolor{gray!10}\llap{\textbf{81.5}} \\
 & & \textcolor{black}{ReForge-DA} & - 
 & 78.9 & \std{0.03} & 79.2 & \std{0.04} & 79.7 & \std{0.04} & 79.1 & \std{0.02} & 84.3 & \std{0.04} & 84.6 & \std{0.04} & \textbf{85.5} & \std{0.06} \\
 & & \textcolor{black}{ReForge-DF}     & - 
 & 82.7 & \std{0.06} & 82.4 & \std{0.03} & 82.3 & \std{0.10} & 82.9 & \std{0.09} & 82.8 & \std{0.03} & 81.7 & \std{0.06} & \textbf{84.3} & \std{0.05} \\
\cmidrule{2-18}
 & \multirow{3}{*}{20}
 & \cellcolor{gray!10}anchor   & \cellcolor{gray!10}91.3 
 & \cellcolor{gray!10} & \cellcolor{gray!10}\llap{55.7}
 & \cellcolor{gray!10} & \cellcolor{gray!10}\llap{60.4}
 & \cellcolor{gray!10} & \cellcolor{gray!10}\llap{64.0}
 & \cellcolor{gray!10} & \cellcolor{gray!10}\llap{72.2}
 & \cellcolor{gray!10} & \cellcolor{gray!10}\llap{76.9}
 & \cellcolor{gray!10} & \cellcolor{gray!10}\llap{76.1}
 & \cellcolor{gray!10} & \cellcolor{gray!10}\llap{\textbf{77.6}} \\
 & & \textcolor{black}{ReForge-DA} & - 
 & 75.6 & \std{0.05} & 75.5 & \std{0.04} & 76.3 & \std{0.03} & 75.9 & \std{0.06} & 82.0 & \std{0.04} & 82.2 & \std{0.06} & \textbf{82.8} & \std{0.05} \\
 & & \textcolor{black}{ReForge-DF}     & - 
 & 80.0 & \std{0.06} & 78.4 & \std{0.06} & 79.2 & \std{0.05} & 79.7 & \std{0.03} & 79.7 & \std{0.05} & 78.0 & \std{0.02} & \textbf{81.5} & \std{0.04} \\

\midrule
\multirow{9}{*}{ViT-L/14}
 & \multirow{3}{*}{8}
 & \cellcolor{gray!10}anchor   & \cellcolor{gray!10}95.8 
 & \cellcolor{gray!10} & \cellcolor{gray!10}\llap{65.1}
 & \cellcolor{gray!10} & \cellcolor{gray!10}\llap{84.8}
 & \cellcolor{gray!10} & \cellcolor{gray!10}\llap{87.0}
 & \cellcolor{gray!10} & \cellcolor{gray!10}\llap{90.0}
 & \cellcolor{gray!10} & \cellcolor{gray!10}\llap{93.0}
 & \cellcolor{gray!10} & \cellcolor{gray!10}\llap{94.0}
 & \cellcolor{gray!10} & \cellcolor{gray!10}\llap{\textbf{94.8}} \\
 & & \textcolor{black}{ReForge-DA} & - 
 & 92.3 & \std{0.01} & 92.8 & \std{0.05} & 93.1 & \std{0.03} & 92.7 & \std{0.03} & 94.4 & \std{0.03} & 94.7 & \std{0.02} & \textbf{95.1} & \std{0.03} \\
 & & \textcolor{black}{ReForge-DF}     & - 
 & 94.4 & \std{0.02} & 94.2 & \std{0.01} & 94.2 & \std{0.03} & 94.0 & \std{0.02} & 94.4 & \std{0.01} & 94.3 & \std{0.04} & \textbf{95.0} & \std{0.02} \\
\cmidrule{2-18}
 & \multirow{3}{*}{14}
 & \cellcolor{gray!10}anchor   & \cellcolor{gray!10}94.3 
 & \cellcolor{gray!10} & \cellcolor{gray!10}\llap{68.5}
 & \cellcolor{gray!10} & \cellcolor{gray!10}\llap{79.4}
 & \cellcolor{gray!10} & \cellcolor{gray!10}\llap{80.2}
 & \cellcolor{gray!10} & \cellcolor{gray!10}\llap{85.5}
 & \cellcolor{gray!10} & \cellcolor{gray!10}\llap{89.1}
 & \cellcolor{gray!10} & \cellcolor{gray!10}\llap{90.5}
 & \cellcolor{gray!10} & \cellcolor{gray!10}\llap{\textbf{91.0}} \\
 & & \textcolor{black}{ReForge-DA} & - 
 & 88.0 & \std{0.01} & 88.3 & \std{0.03} & 88.3 & \std{0.03} & 88.2 & \std{0.02} & 91.6 & \std{0.01} & 91.7 & \std{0.03} & \textbf{92.2} & \std{0.07} \\
 & & \textcolor{black}{ReForge-DF}     & - 
 & 90.9 & \std{0.08} & 90.9 & \std{0.04} & 90.8 & \std{0.02} & 90.6 & \std{0.05} & 91.3 & \std{0.05} & 91.0 & \std{0.06} & \textbf{92.1} & \std{0.06} \\
\cmidrule{2-18}
 & \multirow{3}{*}{20}
 & \cellcolor{gray!10}anchor   & \cellcolor{gray!10}94.7 
 & \cellcolor{gray!10} & \cellcolor{gray!10}\llap{65.4}
 & \cellcolor{gray!10} & \cellcolor{gray!10}\llap{74.0}
 & \cellcolor{gray!10} & \cellcolor{gray!10}\llap{76.7}
 & \cellcolor{gray!10} & \cellcolor{gray!10}\llap{82.8}
 & \cellcolor{gray!10} & \cellcolor{gray!10}\llap{87.7}
 & \cellcolor{gray!10} & \cellcolor{gray!10}\llap{88.4}
 & \cellcolor{gray!10} & \cellcolor{gray!10}\llap{\textbf{90.1}} \\
 & & \textcolor{black}{ReForge-DA} & - 
 & 86.0 & \std{0.02} & 86.2 & \std{0.03} & 86.4 & \std{0.01} & 86.2 & \std{0.01} & 90.9 & \std{0.03} & 90.8 & \std{0.02} & \textbf{91.5} & \std{0.01} \\
 & & \textcolor{black}{ReForge-DF}     & - 
 & 89.8 & \std{0.02} & 89.6 & \std{0.01} & 89.5 & \std{0.01} & 89.4 & \std{0.04} & 90.3 & \std{0.02} & 89.7 & \std{0.02} & \textbf{91.1} & \std{0.11} \\

\bottomrule
\end{tabular}
}
\end{table}

%% file: sections/posterior_calibration.tex
\section{Sampling-based \textcolor{black}{ReForge} for Uncertainty Calibration}
\label{app:posterior-sampling}
\paragraph{Sampling-based \textcolor{black}{ReForge}.}
\textcolor{black}{ReForge uses a MAP point estimate of $\mathbf U$ for computational efficiency. The Bayesian linear regression formulation in Section~\ref{sec:linear_regression} also yields a Gaussian posterior with independent rows, each having mean given by the corresponding row of $\mathbf U_{\mathrm{MAP}}$ and covariance $\mathbf\Sigma_{\mathrm{row}}$:}
\begin{equation}
    \mathbf{U}_{\text{MAP}} = \left( \mathbf{Y}\mathbf{X}^\top + \lambda \mathbf{U}^{(0)} \right) \left( \mathbf{X}\mathbf{X}^\top + \lambda \mathbf{I} \right)^{-1},
    \quad\quad
    \mathbf{\Sigma}_\text{row} = \beta^{-1}\left( \mathbf{X}\mathbf{X}^\top + \lambda \mathbf{I} \right)^{-1},
    \label{eq:sampling-based}
\end{equation}
\textcolor{black}{Here $\mathbf I\in\mathbb R^{d_{\mathrm{in}}\times d_{\mathrm{in}}}$ is the identity matrix, $\mathbf\Sigma_{\mathrm{row}}$ is the covariance matrix for each row of $\mathbf U$, and $\beta>0$ is the Gaussian noise precision. The matrices $\mathbf X$ and $\mathbf Y$ follow the task-wise normalization in Section~\ref{sec:linear_regression}. Once BO selects the regularization strengths and assembly scales, we keep them fixed and sample multiple merged models for uncertainty calibration}~\citep{guo17_calibration}. \textcolor{black}{Varying $\beta$ with $\alpha_m=\beta\lambda_m^\star$ changes the posterior covariance without changing the MAP estimate.}

\textcolor{black}{Specifically, let $\{\theta_{\mathrm{merged}}^{(r)}\}_{r=1}^S$ be $S$ merged models obtained by drawing module task vectors from the Gaussian posteriors and assembling them as $\mathbf W_m^{(r)}=\mathbf W_{\mathrm{pre},m}+s_m^\star\mathbf U_m^{(r)}$. Let $\mathbf z(\mathbf x;\theta)$ denote the logit vector of model $\theta$. The ensemble prediction is the average probability across the $S$ models:}
\begin{equation}
    \textcolor{black}{p(\cdot\mid\mathbf{x})}
    =
    \frac{1}{S}
    \sum_{r=1}^{S}
    \operatorname{softmax}\!\left(\mathbf{z}(\mathbf{x};\theta^{(r)}_{\mathrm{merged}})\right).
    \label{eq:app_covcal_prob_avg}
\end{equation}

\paragraph{Expected Calibration Error.} \textcolor{black}{We measure calibration using ECE with ten equal-width confidence bins:}
\begin{equation}
    \mathrm{ECE}=\sum_{b=1}^{\kappa} \frac{\left|B_{b}\right|}{n}\left|\operatorname{acc}\left(B_{b}\right)-\operatorname{conf}\left(B_{b}\right)\right|,
\end{equation}
\textcolor{black}{Here $\kappa=10$, $B_b$ contains examples whose predicted confidence falls in bin $b$, $|B_b|$ is the number of examples in that bin, and $n$ is the total number of examples. The quantities $\operatorname{acc}(B_b)$ and $\operatorname{conf}(B_b)$ denote the bin's empirical accuracy and mean predicted confidence, respectively.} \textcolor{black}{Lower ECE indicates closer agreement between binned confidence and empirical accuracy.}

\paragraph{Experimental Setup.}
\textcolor{black}{We evaluate sampling-based ReForge on the ViT-B/32 8-task and 20-task benchmarks in the data-assisted setting, using ISO-CTS as the anchor. Based on preliminary experiments, we sample only the final MLP output matrix and keep the remaining modules at their MAP estimates} -- similar observations have also been made by~\citep{kristiadi2020being}. As a baseline, we also perturb the MAP solution of the last MLP output matrix with IID Gaussian noise (controlled by $\sigma_\text{iid}$). \textcolor{black}{Both methods generate $S=10$ sampled models and use prediction averaging in Eq.~\eqref{eq:app_covcal_prob_avg} for final classification. Varying $\beta$ and $\sigma_{\mathrm{iid}}$ changes the trade-off between accuracy and ECE. Figure~\ref{fig:covcal_vitb32} reports the resulting test-set Pareto frontiers; the displayed configurations are not selected using validation data.}

\paragraph{Results and Analysis.}
\textcolor{black}{Figure~\ref{fig:covcal_vitb32} shows that sampling-based ReForge offers a better trade-off between accuracy and ECE than perturbing the MAP model with IID Gaussian noise. The separation between the two Pareto frontiers is more pronounced on the 8-task benchmark than on the 20-task benchmark. These results suggest that the posterior covariance structure is useful for improving uncertainty calibration while retaining most of the MAP model's accuracy.}

%% file: sections/runtime_details.tex
\section{Computational Costs and Runtime}
\label{app:runtime}
\subsection{Complexity Analysis}
\label{app:complexity}
We briefly analyze the costs of the closed-form estimators (Eqs.~\ref{eq:closed_form},~\ref{eq:data_free_closed_form}) and the BO search~\eqref{eq:bilevel_outer}. Let $T$ be the number of tasks, \(M\) be the number of merged 2D modules, and \(\mathbf{U}_m\in\mathbb{R}^{d_{\mathrm{out},m}\times d_{\mathrm{in},m}}\). In the data-assisted setting, with $N$ calibration samples per task, \(\mathbf{X}_m\in\mathbb{R}^{d_{\mathrm{in},m}\times N_{\mathrm{act}}}\) and \(\mathbf{Y}_m\in\mathbb{R}^{d_{\mathrm{out},m}\times N_{\mathrm{act}}}\), where \(N_{\mathrm{act}}=NT\). \textcolor{black}{The matrices $\mathbf X_m$ and $\mathbf Y_m$ include the per-task trace rescaling defined in Section~\ref{sec:linear_regression}.} We omit the one-time forward cost for collecting activations.

\paragraph{Closed-form estimators.}
The data-assisted estimator in Eq.~\eqref{eq:closed_form} first caches \(\mathbf{X}_m\mathbf{X}_m^{\top}\) and \(\mathbf{Y}_m\mathbf{X}_m^{\top}\), which costs
$
\mathcal{O}\!\left(
N_{\mathrm{act}}d_{\mathrm{in},m}^{2}
+
N_{\mathrm{act}}d_{\mathrm{out},m}d_{\mathrm{in},m}
\right).
$
Then, given a module-wise \(\lambda_m\), solving
\begin{equation*}
\left(
\mathbf{X}_m\mathbf{X}_m^{\top}+\lambda_m\mathbf{I}
\right)
\mathbf{U}_m^{\top}
=
\left(
\mathbf{Y}_m\mathbf{X}_m^{\top}+\lambda_m\mathbf{U}_m^{(0)}
\right)^{\top},
\end{equation*}
costs
$
\mathcal{O}\!\left(
 d_{\mathrm{in},m}^{3}
 +
 d_{\mathrm{out},m}d_{\mathrm{in},m}^{2}
\right)
$
by using a dense Cholesky solver~\citep{golub2013matrix}. Similarly, the data-free estimator in Eq.~\eqref{eq:data_free_closed_form} replaces activation statistics with the \textcolor{black}{sum} of \textcolor{black}{per-task trace-normalized task-vector Grams}, whose construction costs \(\mathcal{O}(T d_{\mathrm{out},m}d_{\mathrm{in},m}^{2})\) per module. \textcolor{black}{The resulting linear system has the same per-module solve cost as in the data-assisted setting.} Thus, the cost of the closed-form estimators is insignificant on modern high-performance GPUs.

\paragraph{BO search overhead.}
Algorithm~\ref{alg:overall} performs \(K\) BO trials. The exact $\mathcal{GP}$ update costs \(\mathcal{O}(K^3)\) due to the Cholesky factorization of the $\mathcal{GP}$ kernel matrix~\citep{rasmussen2006gaussian,frazier2018bayesian}. In practice, this overhead is very small compared with repeated validation-set evaluations. Hence, the practical runtime of BO is dominated by evaluating candidate merged models on the validation set.

\subsection{Runtime Breakdown}
\label{app:runtime-breakdown}
\input{tables/runtime}
\textcolor{black}{Table~\ref{tab:runtime_vit_llama} reports ReForge refinement costs after an anchor is available.} The vision experiments with $K\!=\!200$ BO trials are measured on a single NVIDIA RTX \textcolor{black}{6000 Ada} 48GB \textcolor{black}{GPU}, with validation tensors cached locally to reduce I/O overhead. The language experiments with $K\!=\!100$ BO trials report per-worker timings under an 8-GPU parallel setup. Across both vision and language benchmarks, the algorithmic overhead of \textcolor{black}{ReForge} remains small compared with validation-set forward passes, indicating that the main runtime cost comes from model evaluation rather than from the \textcolor{black}{ReForge} update itself.

\subsection{Runtime--Performance Trade-offs}
\label{app:runtime-tradeoff}
\input{tables/grid_bmm_runtime_score}
\textcolor{black}{Figure~\ref{fig:bmm-ablation}}(\textcolor{black}{b}) \textcolor{black}{and Table~\ref{tab:grid_bmm_runtime_score} characterize the additional cost of refining an existing anchor in the data-free setting. On 20-task ViT-B}/\textcolor{black}{32, ReForge refines the ISO-CTS anchor from $77.6\%$ to $80.51\%$ with $K=20$ trials and 5.60 minutes of additional computation; $K=60$ reaches $81.5\%$ in 15.55 additional minutes.}

\textcolor{black}{For Llama-3.2}-\textcolor{black}{3B, $K=20$ costs 28.6 additional minutes and reaches $0.465$, below the TSV anchor's $0.471$; $K=60$ costs 84.8 additional minutes and reaches $0.473$. The language timings in this table are 8-GPU parallel wall-clock times, whereas Table~\ref{tab:runtime_vit_llama} reports per-worker timings; they should not be directly equated.}

%% file: tables/runtime.tex
\begin{table}[t]
\centering
\caption{\textbf{\textcolor{black}{Incremental refinement costs} of \textcolor{black}{ReForge given an existing anchor}.}
The ViT timings are measured on a single NVIDIA \textcolor{black}{RTX 6000 Ada 48GB} GPU, while the Llama timings are reported per worker under an 8-way parallel BO execution.
\emph{Gram} denotes the one-time Gram-matrix construction cost in the data-assisted setting.
\emph{Closed-form} and \textcolor{black}{$\mathcal{GP}$} denote the cumulative costs of closed-form solution and \textcolor{black}{$\mathcal{GP}$} surrogate updates, respectively.
\emph{Search} includes Gram construction, closed-form solution, \textcolor{black}{$\mathcal{GP}$} updates, model assembly, checkpoint loading/book-keeping, and rounding.
\emph{Val. Cost} reports the time spent on validation-set evaluations.}
\label{tab:runtime_vit_llama}

\small
\setlength{\tabcolsep}{4pt} 

\begin{tabular}{@{}lllrrrrrr@{}}
\toprule
Model & Tasks & Setting & Gram & Closed-form & \textcolor{black}{$\mathcal{GP}$} & Search & Val. Cost & Total \\
\midrule
ViT-B/32 & 8  & \textcolor{black}{ReForge-DA} & 0.4 min & 2.5 min  & 0.4 min & 3.3 min  & 17.3 min  & 20.6 min \\
ViT-B/32 & 8  & \textcolor{black}{ReForge-DF}     & --      & 2.2 min  & 0.4 min & 2.9 min  & 17.3 min  & 20.2 min \\
ViT-B/32 & 14 & \textcolor{black}{ReForge-DA} & 0.6 min & 4.5 min  & 0.4 min & 6.2 min  & 22.3 min  & 28.5 min \\
ViT-B/32 & 14 & \textcolor{black}{ReForge-DF}     & --      & 5.2 min  & 0.4 min & 5.6 min  & 22.3 min  & 27.9 min \\
ViT-B/32 & 20 & \textcolor{black}{ReForge-DA} & 0.8 min & 4.6 min  & 0.4 min & 8.0 min  & 41.6 min  & 49.6 min \\
ViT-B/32 & 20 & \textcolor{black}{ReForge-DF}     & --      & 6.9 min  & 0.4 min & 7.3 min  & 41.6 min  & 48.8 min \\
\midrule
ViT-L/14 & 8  & \textcolor{black}{ReForge-DA} & 2.8 min & 8.2 min  & 0.4 min & 17.1 min & 168.4 min & 185.5 min \\
ViT-L/14 & 8  & \textcolor{black}{ReForge-DF}     & --      & 13.9 min & 0.4 min & 14.3 min & 168.4 min & 182.7 min \\
ViT-L/14 & 14 & \textcolor{black}{ReForge-DA} & 5.8 min & 12.1 min & 0.4 min & 18.6 min & 292.1 min & 310.7 min \\
ViT-L/14 & 14 & \textcolor{black}{ReForge-DF}     & --      & 11.3 min & 0.4 min & 12.8 min & 292.1 min & 304.9 min \\
ViT-L/14 & 20 & \textcolor{black}{ReForge-DA} & 9.8 min & 17.0 min & 0.4 min & 31.7 min & 539.0 min & 570.7 min \\
ViT-L/14 & 20 & \textcolor{black}{ReForge-DF}     & --      & 21.5 min & 0.4 min & 21.9 min & 539.0 min & 560.9 min \\
\midrule
Llama-3.2-3B & 5 & \textcolor{black}{ReForge-DA} & 1.0 min & 12.9 min & 0.1 min & 14.5 min & 126.9 min & 141.4 min \\
Llama-3.2-3B & 5 & \textcolor{black}{ReForge-DF}     & --      & 13.4 min & 0.1 min & 13.5 min & 126.9 min & 140.3 min \\
Llama-3.1-8B & 5 & \textcolor{black}{ReForge-DA} & 2.8 min & 38.4 min & 0.1 min & 43.4 min & 190.7 min & 234.1 min \\
Llama-3.1-8B & 5 & \textcolor{black}{ReForge-DF}     & --      & 39.9 min & 0.1 min & 39.9 min & 190.7 min & 230.6 min \\
\bottomrule
\end{tabular}
\end{table}

%% file: tables/grid_bmm_runtime_score.tex
\begin{table}[t]
\centering
\caption{\textcolor{black}{Baseline search costs} and \textcolor{black}{incremental ReForge refinement costs in the} data-free \textcolor{black}{setting}. For ViT-B/32, \textcolor{black}{ReForge} uses ISO-CTS as anchor; for Llama-3.2-3B, \textcolor{black}{ReForge} uses TSV as \textcolor{black}{anchor}. 
Baselines use their full grid-search budgets. 
The ViT runtimes are measured on a single NVIDIA \textcolor{black}{RTX 6000 Ada 48GB} GPU, and the Llama runtimes report 8-GPU parallel wall-clock time. 
Scores denote test accuracies for ViT-B/32 and aggregate benchmark scores for Llama-3.2-3B.}
\label{tab:grid_bmm_runtime_score}
\begin{tabular}{lllrrrr}
\toprule
Model & Tasks & Method & Search Dim. & \#Trials & Runtime & Score \\
\midrule
ViT-B/32 & 20 & TSV & 1 & 30 & 8.4 min & 76.9 \\
ViT-B/32 & 20 & WUDI-Merging & 2 & 24 & 116.1 min & 76.1 \\
ViT-B/32 & 20 & ISO-CTS & 3 & 225 & 91.2 min & 77.6 \\
ViT-B/32 & 20 & \textcolor{black}{ReForge} ($K=20$) & 15 & 20 & 5.60 min & 80.51 \\
ViT-B/32 & 20 & \textcolor{black}{ReForge} ($K=60$) & 15 & 60 & 15.55 min & 81.5 \\
\midrule
Llama-3.2-3B & 5 & TSV & 1 & 30 & 73.4 min & 0.471 \\
Llama-3.2-3B & 5 & WUDI-Merging & 2 & 24 & 418.8 min & 0.461 \\
Llama-3.2-3B & 5 & ISO-CTS & 3 & 225 & 318.0 min & 0.441 \\
Llama-3.2-3B & 5 & \textcolor{black}{ReForge} ($K=20$) & 5 & 20 & 28.6 min & 0.465 \\
Llama-3.2-3B & 5 & \textcolor{black}{ReForge} ($K=60$) & 5 & 60 & 84.8 min & 0.473 \\
\bottomrule
\end{tabular}
\end{table}

%% file: sections/experimental_protocol.tex
\section{Experimental Protocol and Assets}
\label{app:exp-protocol}
\subsection{Vision Benchmarks}
\label{app:vision-benchmarks}
For the \emph{vision} experiments, to ensure a rigorous and fair comparison, we strictly adhere to the benchmarks established by TSV~\citep{gargiulo2025tsv} and ISO-CTS~\citep{marczak2025iso}, adopting their exact datasets and identical training, validation, and test splits. Specifically, our evaluation spans three benchmarks with increasing task diversity. The initial 8-task benchmark comprises Stanford Cars~\citep{krause2013cars}, DTD~\citep{cimpoi2014dtd}, EuroSAT~\citep{helber2019eurosat}, GTSRB~\citep{stallkamp2011gtsrb}, MNIST~\citep{lecun1998mnist}, RESISC45~\citep{cheng2017resisc45}, SUN397~\citep{xiao2016sun397}, and SVHN~\citep{netzer2011svhn}. The 14-task benchmark extends this setting by adding CIFAR-100~\citep{krizhevsky2009cifar}, STL-10~\citep{coates2011stl10}, Flowers102~\citep{nilsback2008flowers}, Oxford-IIIT Pets~\citep{parkhi2012pets}, PCAM~\citep{veeling2018pcam}, and FER2013~\citep{goodfellow2013fer2013}. Finally, the largest benchmark encompasses 20 tasks by further incorporating EMNIST~\citep{cohen2017emnist}, CIFAR-10~\citep{krizhevsky2009cifar}, Food101~\citep{bossard2014food101}, Fashion-MNIST~\citep{xiao2017fashionmnist}, Rendered SST-2~\citep{socher2013sst}, and KMNIST~\citep{clanuwat2018kmnist}. Table~\ref{tab:merged_performance_comparison} in the main paper reports the performance of the merged models averaged across the 8, 14, and 20 tasks, while the ``mean $\pm$ std'' results and the per-task breakdown radar charts are provided in Appendices~\ref{app:detailed-aggregate} and~\ref{app:vit-pertask}, respectively.

\subsection{Language Benchmarks}
\label{app:language-benchmarks}
\begin{table}[H]
\centering
\caption{Summary of training and evaluation datasets for natural language processing experiments. Official standard splits are adopted for all benchmarks, with the exception of the GSM8k validation set (500 samples from the training set) and datasets marked with * (custom 1/3-2/3 splits).}
\label{tab:unified_datasets_nlp}
\small
\rowcolors{2}{white}{gray!10} 
\begin{tabularx}{\linewidth}{
    >{\raggedright\arraybackslash}p{1.65cm} 
    >{\raggedright\arraybackslash}X      
    >{\raggedright\arraybackslash}X      
    >{\raggedright\arraybackslash}X      
    >{\raggedright\arraybackslash}p{1.65cm} 
}
\toprule
\textbf{Category} & \textbf{Training} & \textbf{Validation} & \textbf{Test} & \textbf{Metric} \\
\midrule
Instruction & TULU-3 Persona~\citep{lambert2024tulu} & IFEval*~\citep{zhou2023instruction} & IFEval*~\citep{zhou2023instruction} & Prompt Acc. \\
Mathematics & DART~\citep{tong2024dart}, Numina\allowbreak Math~\citep{li2024numinamath} & GSM8k~\citep{cobbe2021training} & GSM8k~\citep{cobbe2021training} & EM (8-shot) \\
Multilingual & Aya~\citep{singh2024aya} & M\_MMLU~\citep{lai2023okapi} (fr, es, de, ru) & M\_MMLU~\citep{lai2023okapi} (fr, es, de, ru) & Accuracy \\
Coding & Magicoder~\citep{wei2023magicoder} & MBPP~\citep{austin2021program} & HumanEval+~\citep{liu2023evalplus}, MBPP+~\citep{liu2023evalplus} & Pass@1 \\
Safety & WildGuard~\citep{han2024wildguard}, Wild\allowbreak Jailbreak~\citep{jiang2024wildteaming} & Harm\allowbreak Bench*~\citep{mazeika2024harmbench}, XSTest*~\citep{rottger2023xstest} & Harm\allowbreak Bench*~\citep{mazeika2024harmbench}, XSTest*~\citep{rottger2023xstest} & \textcolor{black}{$1-\mathrm{ASR}$} / Acc. \\
\bottomrule
\end{tabularx}
\end{table}
For the NLP experiments, Table~\ref{tab:unified_datasets_nlp} summarizes the training, validation, and test splits, together with the corresponding evaluation metrics. \textcolor{black}{For the coding category, we use MBPP for validation after removing 13 examples whose normalized problem statements exactly overlap with examples in the MBPP}+ \textcolor{black}{test set. The resulting validation and test sets are disjoint.} 

\paragraph{\textcolor{black}{Language validation and scoring.}}
\label{sec:language_scoring}
\textcolor{black}{For Llama, benchmark scores are expressed on a $[0,1]$ scale and aggregated into five category scores. Instruction following uses IFEval prompt accuracy, and mathematics uses GSM8k exact match. Multilingual performance is the arithmetic mean of M}\_\textcolor{black}{MMLU accuracies for French, Spanish, German, and Russian. Coding uses MBPP pass}@\textcolor{black}{1 for validation and the arithmetic mean of HumanEval}+ \textcolor{black}{and MBPP}+ \textcolor{black}{pass}@\textcolor{black}{1 for testing. Safety uses the harmonic mean}
\begingroup\color{black}\[
S_{\mathrm{Safety}}=H(1-\mathrm{ASR}_{\mathrm{HarmBench}},\mathrm{Acc}_{\mathrm{XSTest}}),
\qquad H(a,b)=\frac{2ab}{a+b},
\]\endgroup
\textcolor{black}{with $H(0,0)=0$. The final score is the equally weighted arithmetic mean of the five category scores, $S_{\mathrm{overall}}=\frac{1}{5}\sum_{c=1}^{5}S_c$. Validation scores guide BO, while the corresponding held-out test scores are used for final reporting in Table~\ref{tab:nlp_main}.}

\subsection{Licenses and Asset Usage}
\label{app:licenses}
We use existing public assets, including pretrained backbones, task-specific
checkpoints, benchmark datasets, evaluation suites, and baseline
implementations. These assets are described above, and
their original creators are credited through the corresponding references.

For vision experiments, we follow public model-merging benchmark settings based
on ViT backbones and task-specific expert checkpoints. For language experiments,
we use Llama-3.2-3B and Llama-3.1-8B backbones under the corresponding Meta
Llama Community License Agreements and acceptable-use policy. We do not
redistribute third-party checkpoints or model weights except where permitted by
their original licenses.

All datasets and evaluation suites are used only for research, calibration,
validation, and evaluation. Their source URLs, versions when applicable, and
license or source-term information are provided in the supplementary asset
manifest. Safety benchmarks may contain harmful or adversarial prompts and are
used only for safety research and evaluation.

Baseline methods and third-party libraries are used according to their original
papers, public repositories, and licenses. We retain required copyright notices,
citations, and license files where applicable. We do not introduce new scraped
datasets or human-subject data. All copyrights remain with the original asset
owners, and all external assets are used in compliance with their respective
licenses and terms of use.

%% file: sections/radar_results.tex
\section{Radar Charts: Per-Task Breakdowns}
\label{app:vit-pertask}
Tables~\ref{tab:merged_performance_comparison} and~\ref{tab:nlp_main} report the performance of merged models averaged across multiple tasks. However, such summaries may hide task-specific trade-offs. To address this issue, Figures~\ref{fig:radar_vitb32_part1}--\ref{fig:radar_vitl14_part2} provide vision per-task radar charts for all evaluated ViT settings. Similarly, Figures~\ref{fig:radar_llama_part1}--\ref{fig:radar_llama_part2} provide language per-task radar charts for all evaluated Llama settings.

\input{sections/radar_appendix}

%% file: sections/radar_appendix.tex
\newcommand{\radarlegendline}[2]{\textcolor[HTML]{#1}{\rule{8mm}{0.75pt}}\hspace{0.8mm}{\small #2}}
\newcommand{\radarlegend}{%
    \radarlegendline{D62728}{Anchor}\hspace{8mm}%
    \radarlegendline{1F77B4}{\textcolor{black}{ReForge-DA}}\hspace{8mm}%
    \radarlegendline{2CA02C}{\textcolor{black}{ReForge-DF}}\\%
}
\newcommand{\radarinclude}[1]{\raisebox{-.5\height}{\includegraphics[width=0.3\textwidth, keepaspectratio]{#1}}}
\newcommand{\radarlabel}[2]{\raisebox{-.5\height}{\makebox[0.09\textwidth][c]{\shortstack[c]{\scriptsize\textbf{#1}\\\scriptsize (#2)}}}}
\newcommand{\radarheader}{& \textbf{8 Tasks} & \textbf{14 Tasks} & \textbf{20 Tasks} \\[0.12cm]}

\begin{figure}[H]
    \centering
    \radarlegend
    \vspace{0.12cm}

    \setlength{\tabcolsep}{0pt}
    \begin{tabular}{c c c c}
        \radarheader

        \radarlabel{Pretrained}{0--100} &
        \radarinclude{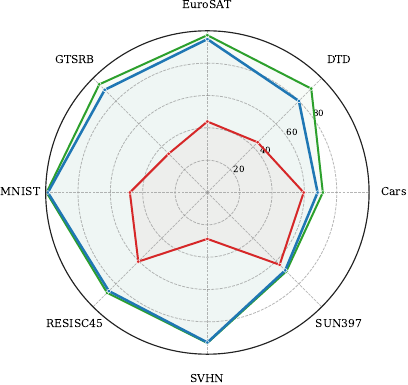} &
        \radarinclude{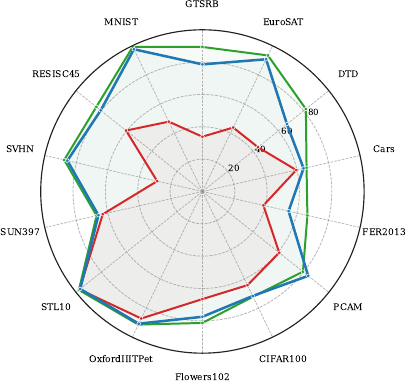} &
        \radarinclude{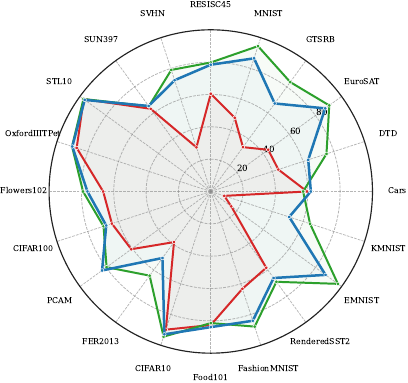} \\[0.08cm]

        \radarlabel{TA}{10--100} &
        \radarinclude{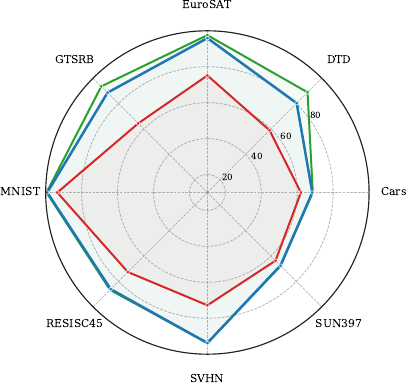} &
        \radarinclude{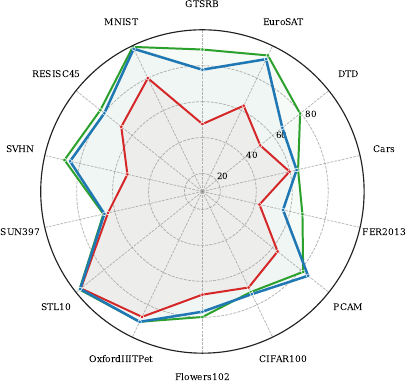} &
        \radarinclude{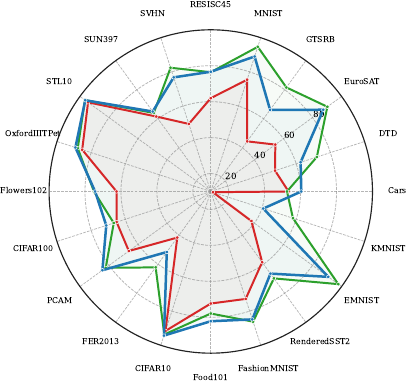} \\[0.08cm]

        \radarlabel{TIES}{10--100} &
        \radarinclude{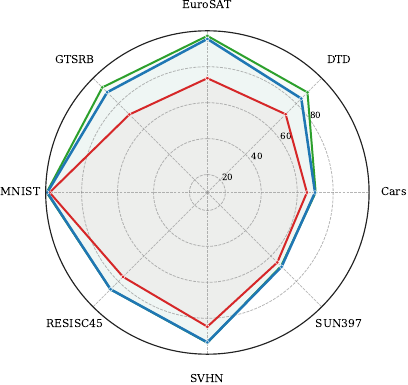} &
        \radarinclude{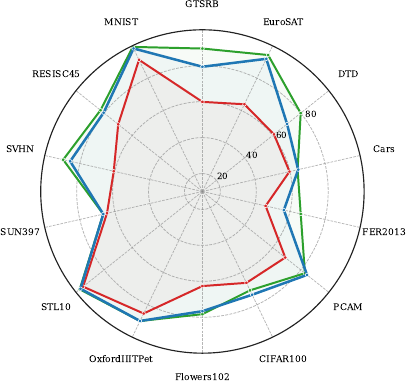} &
        \radarinclude{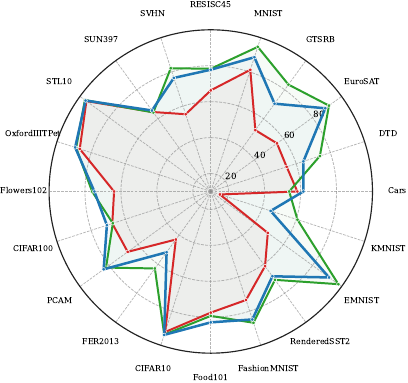} \\[0.08cm]

        \radarlabel{RegMean}{20--100} &
        \radarinclude{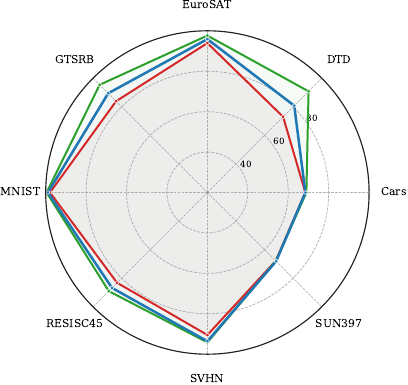} &
        \radarinclude{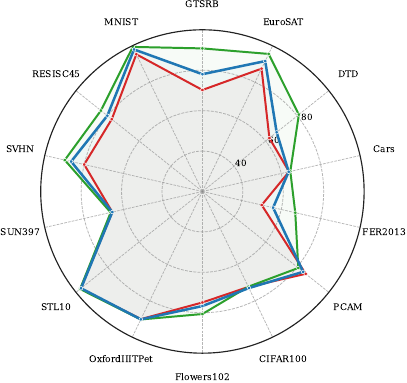} &
        \radarinclude{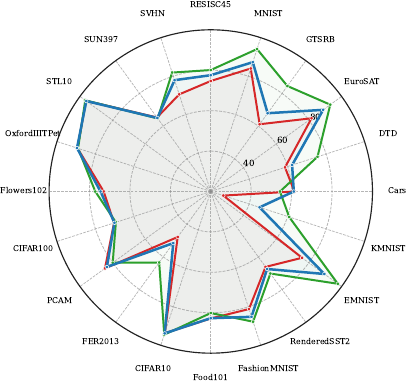} \\[0.08cm]

        \radarlabel{TSV}{50--100} &
        \radarinclude{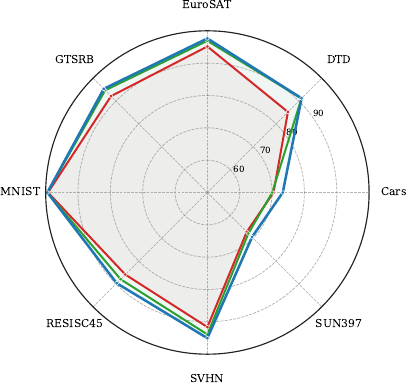} &
        \radarinclude{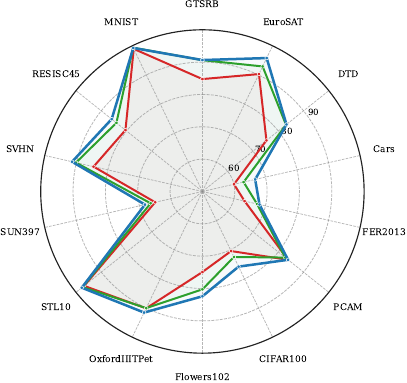} &
        \radarinclude{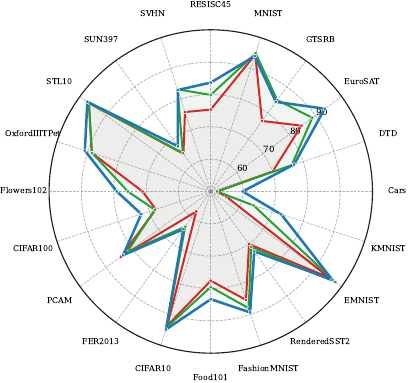}
    \end{tabular}
\vspace{10pt}
\caption{\textbf{Radar charts: ViT-B/32 per-task breakdowns} (corresponding to Table~\ref{tab:merged_performance_comparison}). The radial axis limits (min--max) are annotated on the left.}
\label{fig:radar_vitb32_part1}
\end{figure}

\newpage

\begin{figure}[H]
    \centering
    \radarlegend
    \vspace{0.12cm}

    \setlength{\tabcolsep}{0pt}
    \begin{tabular}{c c c c}
        \radarheader

        \radarlabel{WUDI}{50--100} &
        \radarinclude{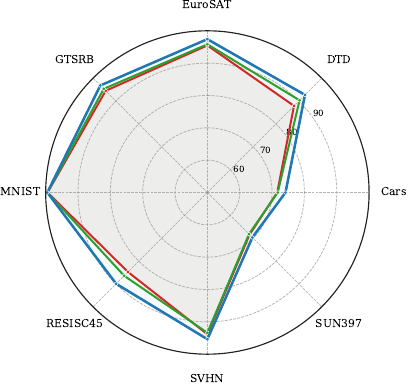} &
        \radarinclude{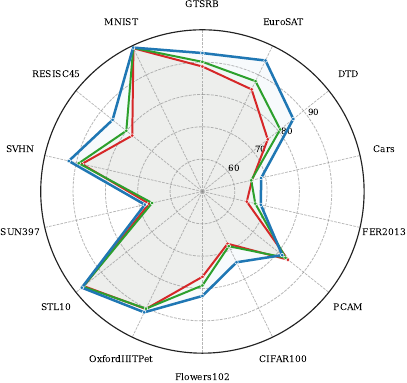} &
        \radarinclude{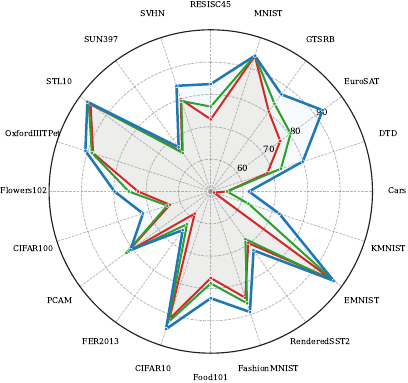} \\[0.08cm]

        \radarlabel{ISO-CTS}{50--100} &
        \radarinclude{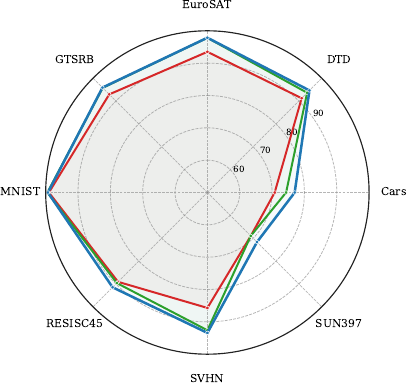} &
        \radarinclude{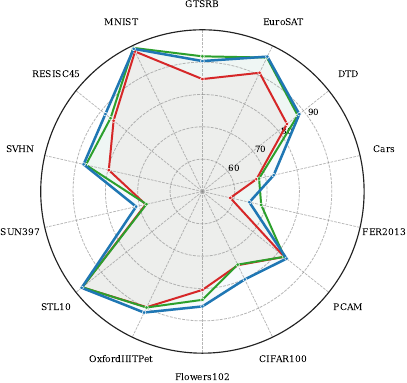} &
        \radarinclude{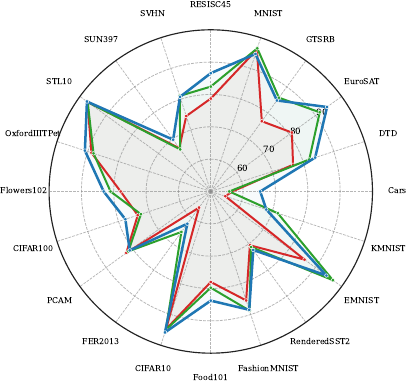}
    \end{tabular}

    \vspace{10pt}
    \caption{\textbf{Radar charts: ViT-B/32 per-task breakdowns} (corresponding to Table~\ref{tab:merged_performance_comparison}). Continued from Figure~\ref{fig:radar_vitb32_part1} on the remaining anchor models: WUDI and ISO-CTS.}
\label{fig:radar_vitb32_part2}
    \vspace{1cm}
    \rule{\textwidth}{0.5pt}
    \vspace{0.45cm}

    \begin{tabular}{c c c c}
        \radarheader

        \radarlabel{Pretrained}{10--100} &
        \radarinclude{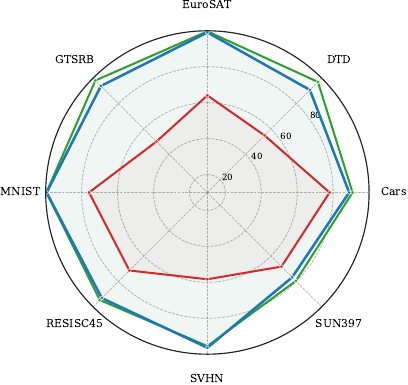} &
        \radarinclude{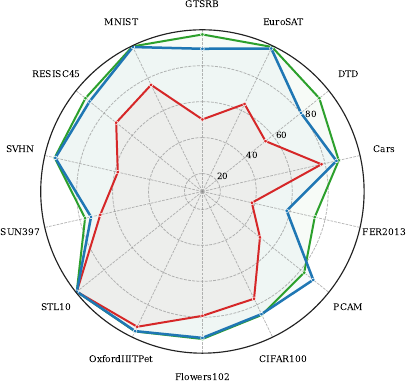} &
        \radarinclude{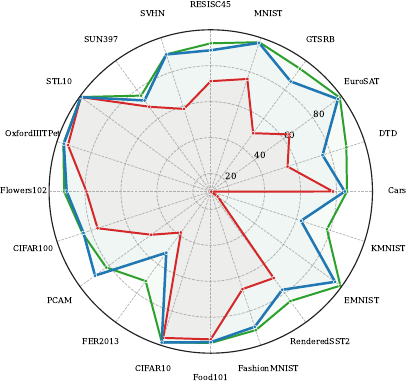} \\[0.08cm]

        \radarlabel{TA}{10--100} &
        \radarinclude{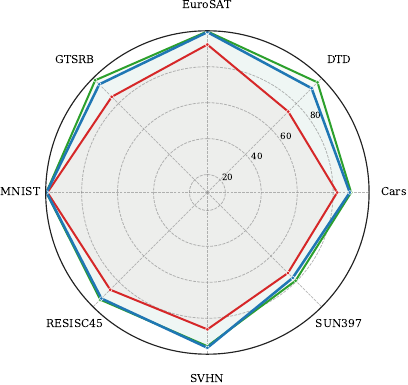} &
        \radarinclude{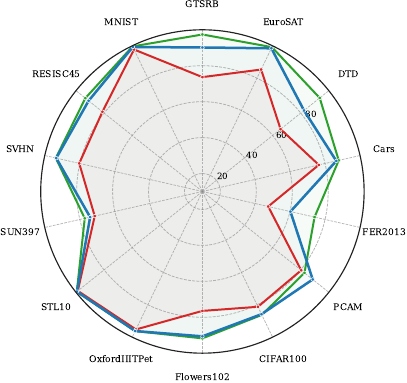} &
        \radarinclude{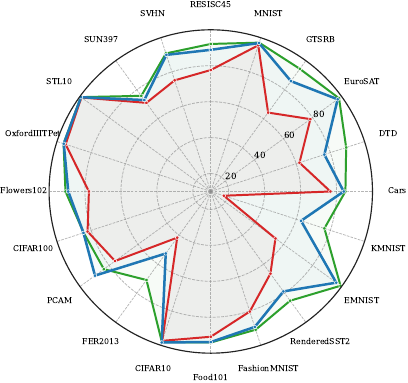}
    \end{tabular}

    \vspace{10pt}
    \caption{\textbf{Radar charts: ViT-L/14 per-task breakdowns} (corresponding to Table~\ref{tab:merged_performance_comparison}). Layout, legend, and row-specific axis scaling follow Figure~\ref{fig:radar_vitb32_part1}.}
\label{fig:radar_vitl14_part1}
\end{figure}

\newpage

\begin{figure}[H]
    \centering
    \radarlegend
    \vspace{0.12cm}

    \setlength{\tabcolsep}{0pt}
    \begin{tabular}{c c c c}
        \radarheader

        \radarlabel{TIES}{30--100} &
        \radarinclude{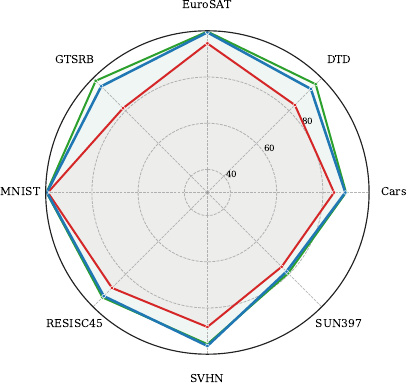} &
        \radarinclude{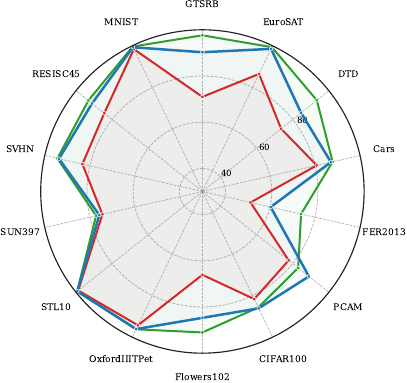} &
        \radarinclude{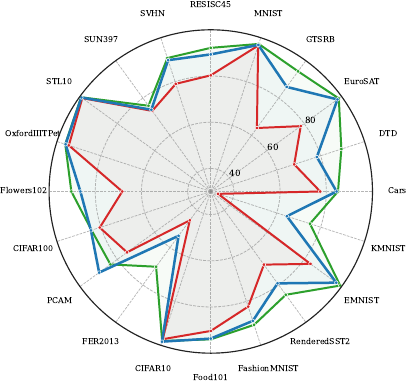} \\[0.08cm]

        \radarlabel{RegMean}{40--100} &
        \radarinclude{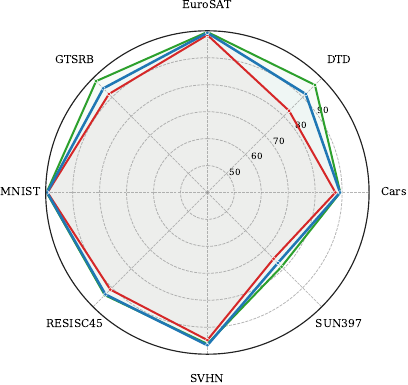} &
        \radarinclude{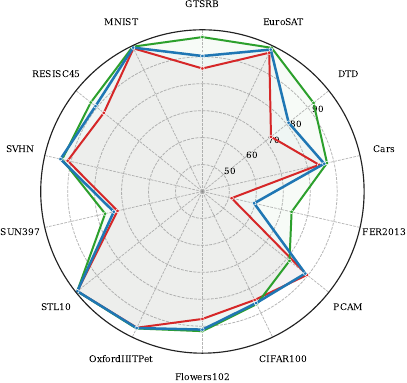} &
        \radarinclude{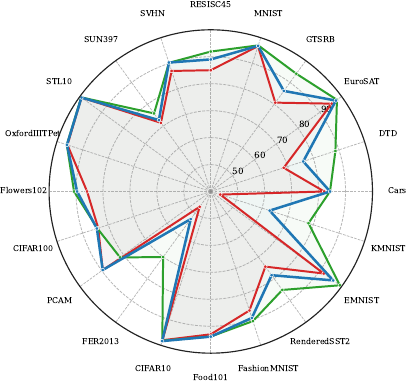} \\[0.08cm]

        \radarlabel{TSV}{60--100} &
        \radarinclude{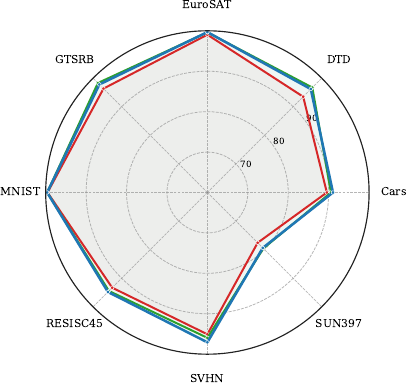} &
        \radarinclude{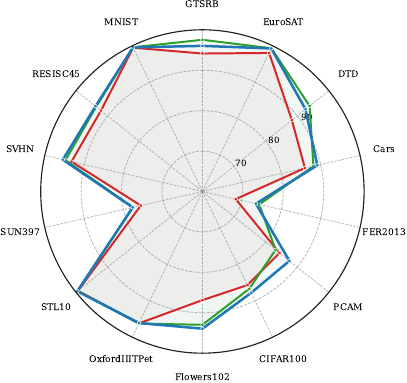} &
        \radarinclude{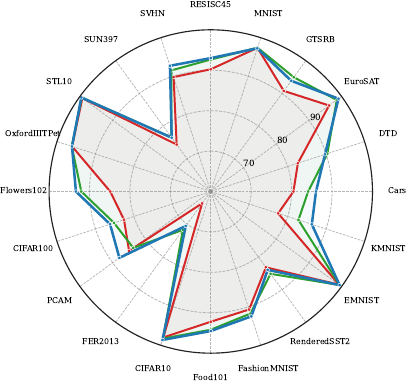} \\[0.08cm]

        \radarlabel{WUDI}{60--100} &
        \radarinclude{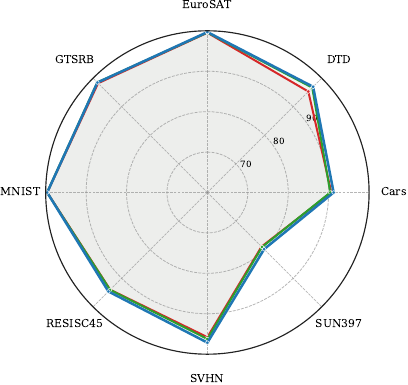} &
        \radarinclude{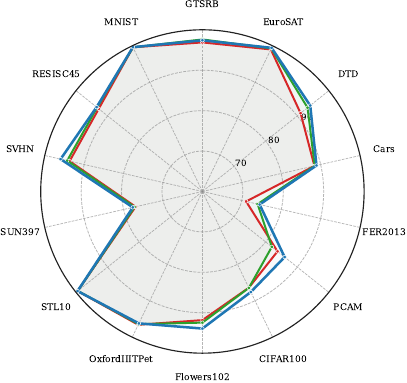} &
        \radarinclude{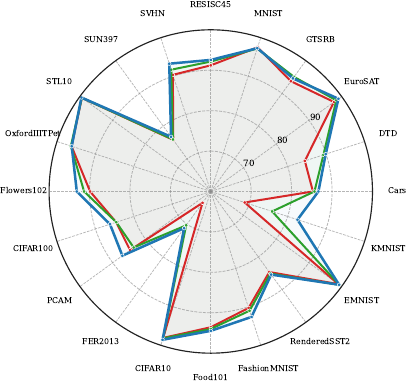} \\[0.08cm]

        \radarlabel{ISO-CTS}{60--100} &
        \radarinclude{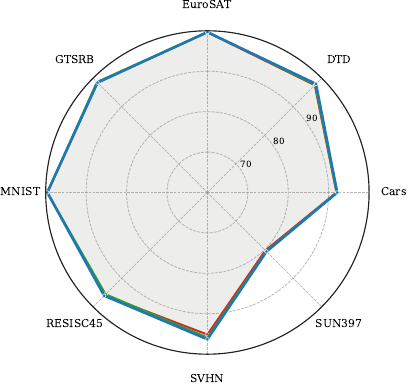} &
        \radarinclude{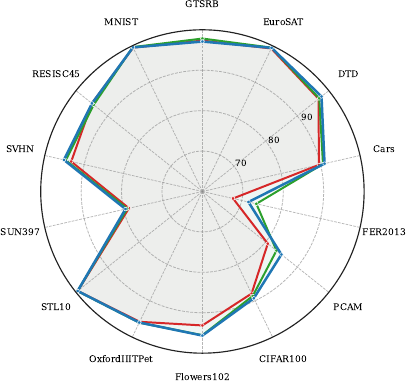} &
        \radarinclude{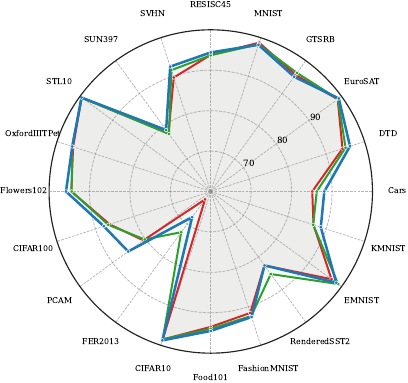}
    \end{tabular}
    
    \vspace{10pt}
    \caption{\textbf{Radar charts: ViT-L/14 per-task breakdowns} (corresponding to Table~\ref{tab:merged_performance_comparison}). Continued from Figure~\ref{fig:radar_vitl14_part1} on the remaining anchor models.}
\label{fig:radar_vitl14_part2}
\end{figure}

\clearpage

\newcommand{\llamaradarinclude}[1]{\raisebox{-.5\height}{\makebox[0.39\textwidth][c]{\includegraphics[width=0.34\textwidth, keepaspectratio]{#1}}}}
\newcommand{\llamaradarlabel}[2]{\raisebox{-.5\height}{\makebox[0.12\textwidth][c]{\shortstack[c]{\scriptsize\textbf{#1}\\\scriptsize (#2)}}}}
\newcommand{\llamaradarheader}{& \makebox[0.39\textwidth][c]{\textbf{Llama-3.2-3B}} & \makebox[0.39\textwidth][c]{\textbf{Llama-3.1-8B}} \\[0.25cm]}

\begin{figure}[H]
    \centering
    \radarlegend
    \vspace{0.12cm}

    \setlength{\tabcolsep}{0pt}
    \begin{tabular}{c c c}
        \llamaradarheader

        \llamaradarlabel{Pretrained}{0--80} &
        \llamaradarinclude{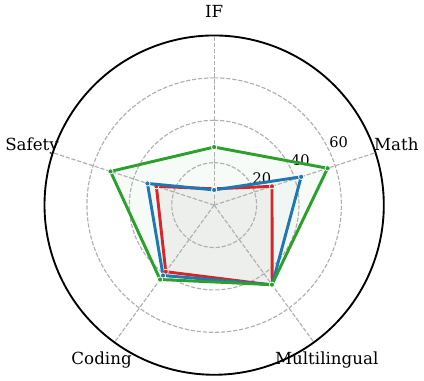} &
        \llamaradarinclude{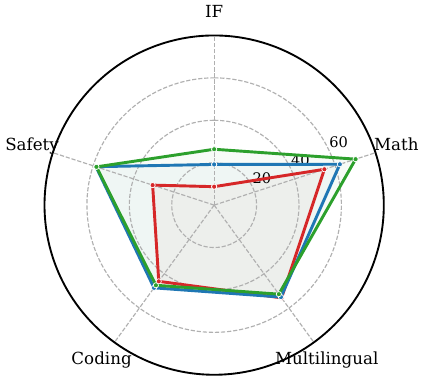} \\[0.8cm]

        \llamaradarlabel{TA}{20--80} &
        \llamaradarinclude{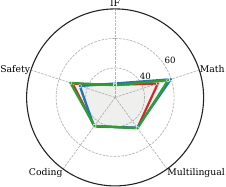} &
        \llamaradarinclude{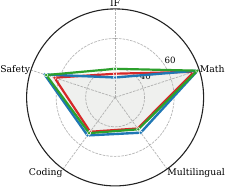} \\[0.8cm]

        \llamaradarlabel{TIES}{20--80} &
        \llamaradarinclude{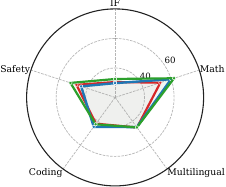} &
        \llamaradarinclude{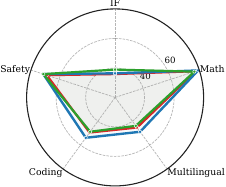} \\[0.8cm]

        \llamaradarlabel{RegMean}{10--80} &
        \llamaradarinclude{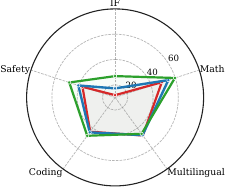} &
        \llamaradarinclude{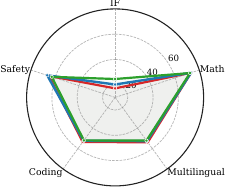}
    \end{tabular}

    \vspace{10pt}
    \caption{\textbf{Radar charts: Llama per-task breakdowns} (corresponding to Table~\ref{tab:nlp_main}). The radial axis limits (min-max) are annotated on the left. ``IF'' denotes Instruction Following.}
\label{fig:radar_llama_part1}
\end{figure}

\newpage

\begin{figure}[H]
    \centering
    \radarlegend
    \vspace{0.12cm}

    \setlength{\tabcolsep}{0pt}
    \begin{tabular}{c c c}
        \llamaradarheader

        \llamaradarlabel{TSV}{30--80} &
        \llamaradarinclude{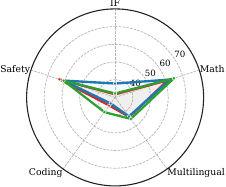} &
        \llamaradarinclude{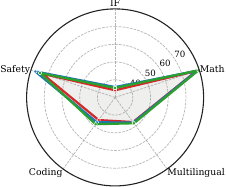} \\[0.8cm]

        \llamaradarlabel{WUDI}{20--80} &
        \llamaradarinclude{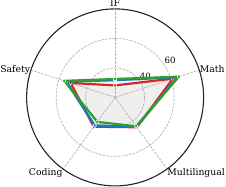} &
        \llamaradarinclude{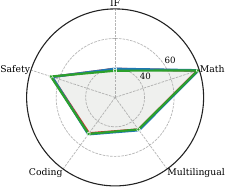} \\[0.8cm]

        \llamaradarlabel{ISO-CTS}{20--80} &
        \llamaradarinclude{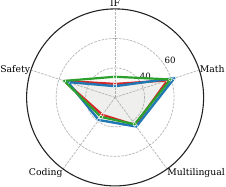} &
        \llamaradarinclude{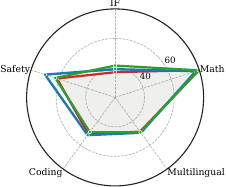}
    \end{tabular}

    \vspace{10pt}
    \caption{\textbf{Radar charts: Llama per-task breakdowns} (corresponding to Table~\ref{tab:nlp_main}). Continued from Figure~\ref{fig:radar_llama_part1} on the remaining anchor models. ``IF'' denotes Instruction Following.}
\label{fig:radar_llama_part2}
\end{figure}